\documentclass[10 pt, conference]{ieeeconf}

\IEEEoverridecommandlockouts

\usepackage{hyperref}
\usepackage{graphics} 
\usepackage{graphicx}
\usepackage{amsmath} 
\usepackage[ruled, vlined, linesnumbered]{algorithm2e}
\usepackage{bm}
\usepackage{subfig}
\usepackage{float}
\usepackage{amssymb}
\usepackage[font={footnotesize}]{caption}
\usepackage[usenames, dvipsnames]{color}
\newtheorem{defi}{Definition}
\newtheorem{theorem}{Theorem}[section]
\newtheorem{lemma}[theorem]{Lemma}
\newtheorem{proposition}[theorem]{Proposition}

\newcommand{\triangulation}{\mathcal{T}}
\newcommand{\dualgraph}{\triangulation^*}
\newcommand{\dualv}{V^*}
\newcommand{\duale}{E^*}
\newcommand{\calV}{\mathcal V_\tau}
\newcommand{\calT}{\mathcal T_\tau}
\newcommand{\calP}{\mathcal P_\tau}
\newcommand{\calTstar}{\mathcal T^*_\tau}
\newcommand{\reals}{\mathbb R}
\newcommand{\dthresh}{D_{thresh}}
\newcommand{\position}{p^\tau_i}
\newcommand{\velocity}{v^\tau_i}
\newcommand{\nt}{n_t}
\newcommand{\ntau}{n_\tau}
\newcommand{\bff}{\mathbf{f}}

\mathchardef\mhyphen="2D

\title{\LARGE \bf
Dynamic Channel: A Planning Framework for Crowd Navigation
}
\author{Chao Cao$^{1}$, Peter Trautman$^{2}$ and Soshi Iba$^{2}$
\thanks{$^{1}$Chao Cao is with Robotics Institute,
        Carnegie Mellon University, Pittsburgh, PA 15213, USA
        {\tt\small ccao1@andrew.cmu.edu}}%
\thanks{$^{2}$Peter Trautman and Soshi Iba with Honda Research Institute,
        Mountain View, CA 94043, USA
        {\tt\small \{ptrautman, siba\}@honda-ri.com}}%
}

\begin{document}

\maketitle
\thispagestyle{empty}
\pagestyle{empty}

\begin{abstract}
Real-time navigation in dense human environments is a challenging problem in robotics. Most existing path planners fail to account for the dynamics of pedestrians because introducing time as an additional dimension in search space is computationally prohibitive. Alternatively, most local motion planners only address imminent collision avoidance and fail to offer long-term optimality. In this work, we present an approach, called Dynamic Channels, to solve this global to local quandary. Our method combines the high-level topological path planning with low-level motion planning into a complete pipeline. By formulating the path planning problem as graph-searching in the triangulation space, our planner is able to explicitly reason about the obstacle dynamics and capture the environmental change efficiently. We evaluate efficiency and performance of our approach on public pedestrian datasets and compare it to a state-of-the-art planning algorithm for dynamic obstacle avoidance.  Completeness proofs are provided in the supplement at \href{http://caochao.me/files/proof.pdf}{http://caochao.me/files/proof.pdf}.  An extended version of the paper is available on arXiv.
\end{abstract}

\section{Introduction}
\noindent Autonomously navigating through crowded human environments is a cornerstone functionality for many mobile robot applications  (e.g., service, safety, and logistic robotics).  However, crowd navigation requires safe operation of the robot in close proximity to agile pedestrians: a competent crowd planner needs to efficiently generate feasible, optimal, and even socially compliant trajectories.  Most existing approaches focus on either global or local optimality.  For instance, sampling based motion planners often assume a static environment because adding a time dimension exponentially increases computation.  Search-based, optimization-based and geometry-based planning algorithms are subject to the same computational shortcoming.  Alternatively, local planners attend to agent dynamics by seeking collision-free paths over short time horizons but neglect global optimality.
\begin{figure}[!t]
	\centering
	\subfloat[]{\includegraphics[width=0.5\linewidth]{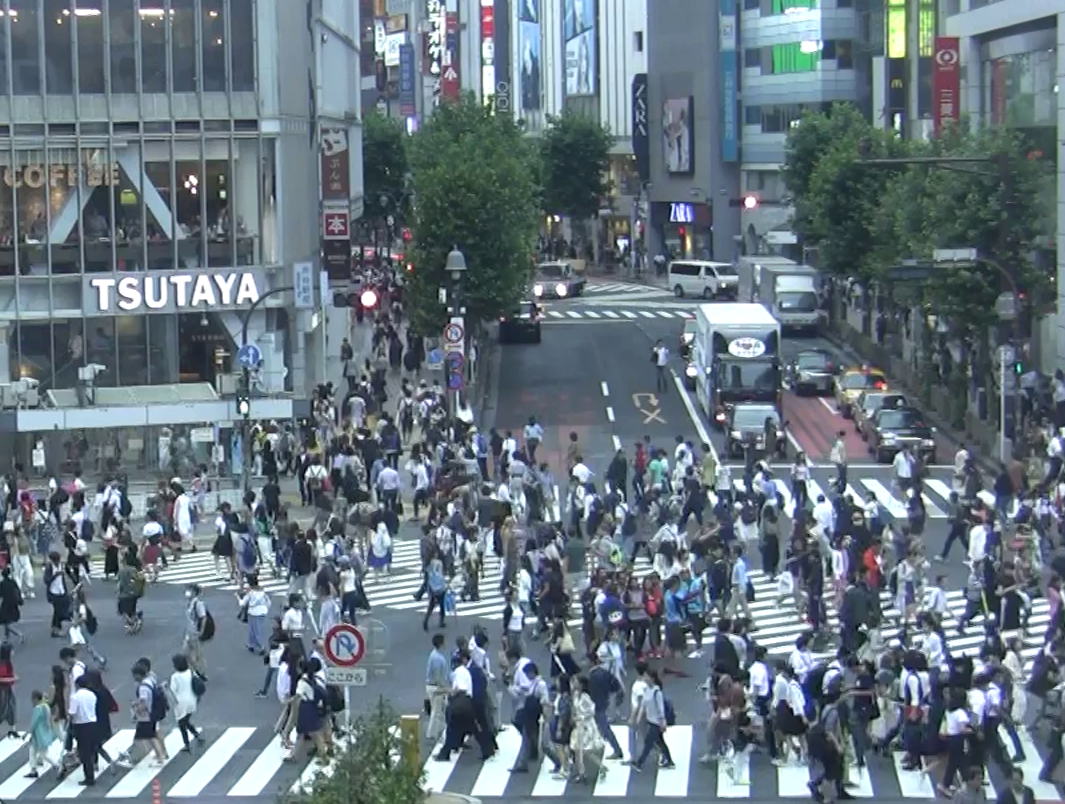}\label{fig:crowds}}
	\subfloat[]{\includegraphics[width=0.5\linewidth]{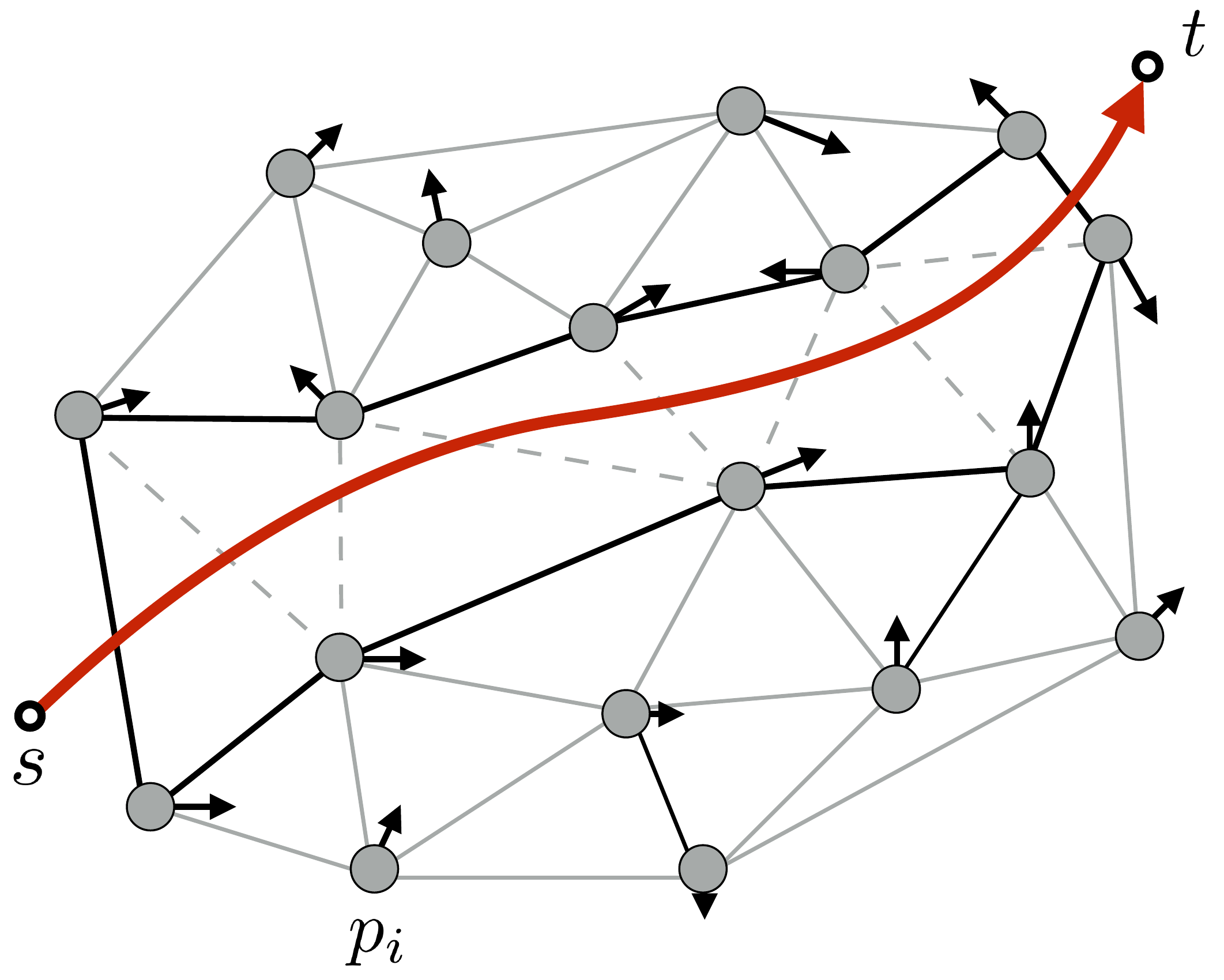}\label{fig:abstraction}}
	\caption{(a) A crowded crosswalk in Japan, among which the navigation of a robot will be challenging without an efficient planner. (b) The environment abstraction used in our framework. Pedestrians are shown as circles with arrows indicating the moving direction and speed. The "Dynamic Channel" and a safe trajectory inside are shown by bold line segments and the red curve with an arrow.}
	\label{fig:crowds_and_abstraction}
	\vspace{-15pt}
\end{figure}

We introduce \emph{dynamic channels}, a novel crowd navigation architecture combining global optimality and a heuristic for efficiently managing agent dynamics. Our method abstracts the environment using Delaunay Triangulation and then projects obstacle dynamics onto the triangulation.  A modified A* algorithm then computes the optimal path.  Empirically, performance and efficiency are validated with thousands of real world pedestrian datasets.  Our approach outperforms state of the art methods by a significant margin for task completion while remaining real-time computable even for large numbers of dynamic obstacles.

\subsection{Local and Global Planning}
\emph{Local planning} seeks to optimize near-term objective functions; classic approaches include the Dynamic Window Approach~\cite{dynamic-window}, Inevitable Collision States~\cite{ics}, and Velocity Obstacles \cite{vos}.  Variants of Velocity Obstacles~\cite{rvo,VanDenBerg2010,snape-hrvo,VanDenBerg2011} solve the multi-agent system problem for reciprocal collision avoidance by assuming that each agent adopts the same navigation strategy (e.g., each agent passes to the right).  However, human behavior is probabilistic, and so assuming fixed actions can lead to unsafe robot behaviors. However, even with guaranteed collision-free local actions, local planning is insufficient, since such trajectories may not guarantee e.g. the shortest traveling distance.

\emph{Global planning} searches for feasible paths connecting initial and goal states of the robot. Most traditional sampling based algorithms including Rapidly-exploring Random Tree (RRT)~\cite{Kuffner2000} and Probabilistic roadmap methods (PRM)~\cite{Kavraki1996,Kavraki1998} assume the environment to be static when planning; replanning is employed to account for dynamics. To reduce unnecessary computation~\cite{Zucker2007} proposed a dynamic RRT variant by reusing branches from previous search trees. In~\cite{Connell2017} a replanning scheme based on RRT*~\cite{rrt*} dealt with unknown dynamic obstacles blocking the robot's path.

Many global planners decouple static and dynamic obstacles.  An intuitive strategy relies on a local reactive collision avoidance system to deal with local disturbances. Separating static and dynamic obstacles can also be achieved by hierarchical planning and the combination of potential fields: ~\cite{JurP.vandenBergandMarkH.Overmars2005,VanDenBerg2007} decouples high and low level planning with a two level-search.  Globally, a roadmap is built upon static obstacles while locally a collision free trajectory is obtained. In~\cite{svenstrup-human}, a dynamic pedestrian potential field selects the traversal and safety optimal trajectory.  The same problem also exists in other planning methods.  For Elastic Band~\cite{Quinlan1993}, Timed Elastic Band~\cite{Rosmann2013} and spline-based planners~\cite{Lau2009}, the initial guess determines path partition, which is based on the static assumption.  Ultimately, treating static and dynamic obstacles separately has the advantage of being efficient and probabilistically complete. However, sub-optimal paths can result because obstacle dynamics affects global optimality.
\begin{figure*}[!h]
\centering
\subfloat[]{\includegraphics[width=0.25\linewidth]{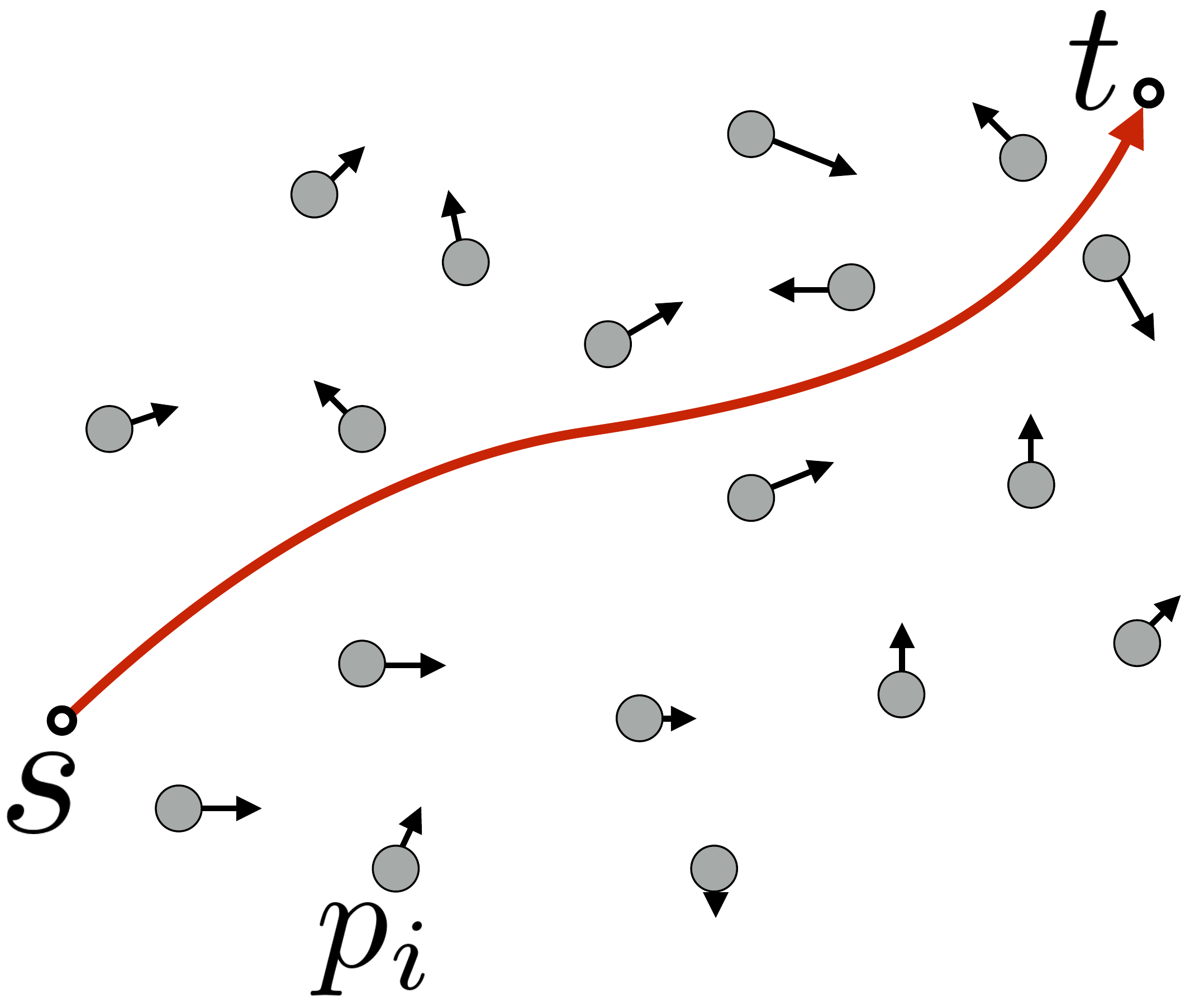}\label{fig:path}}
\subfloat[]{\includegraphics[width=0.25\linewidth]{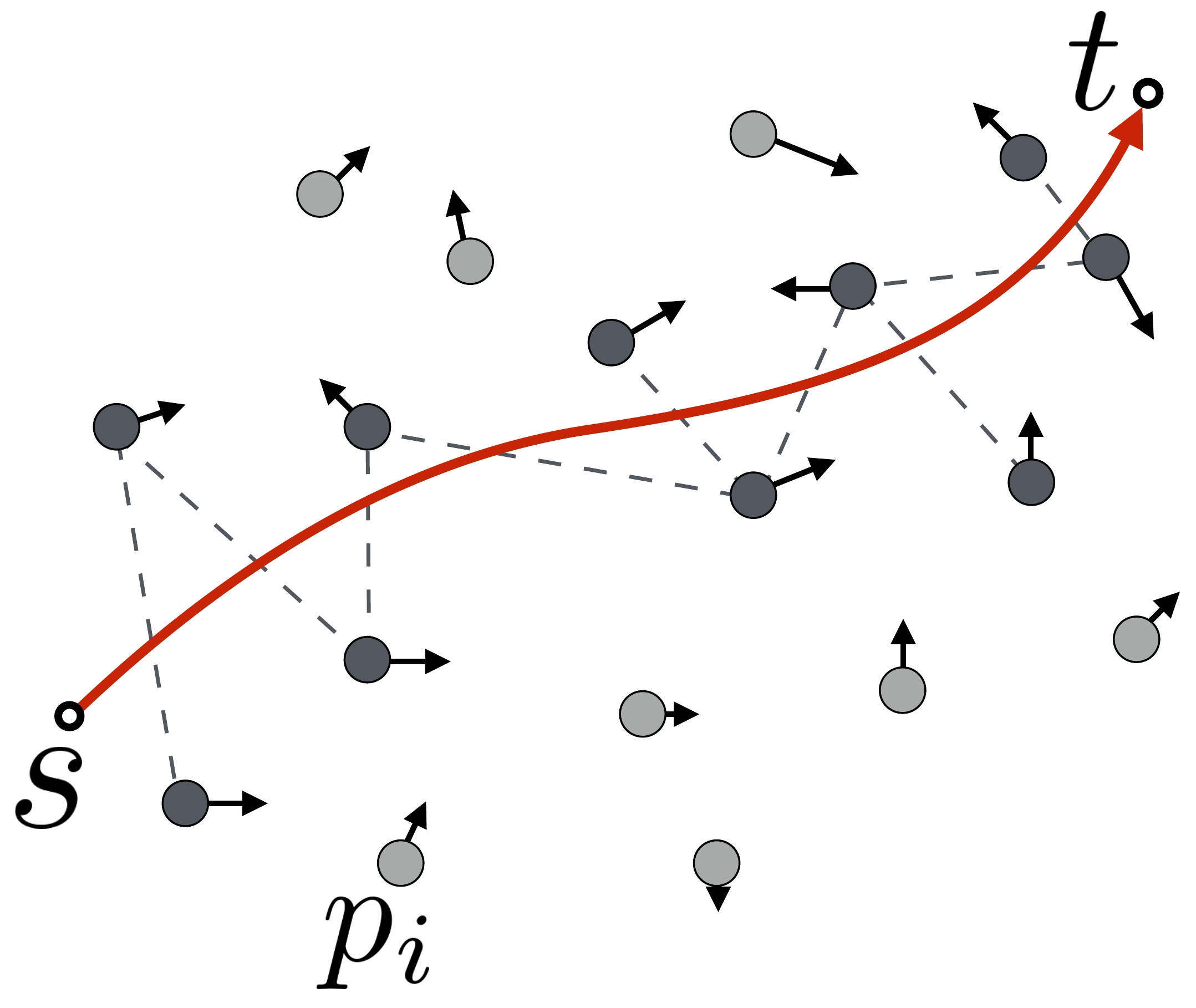}\label{fig:neighbors}}
\subfloat[]{\includegraphics[width=0.25\linewidth]{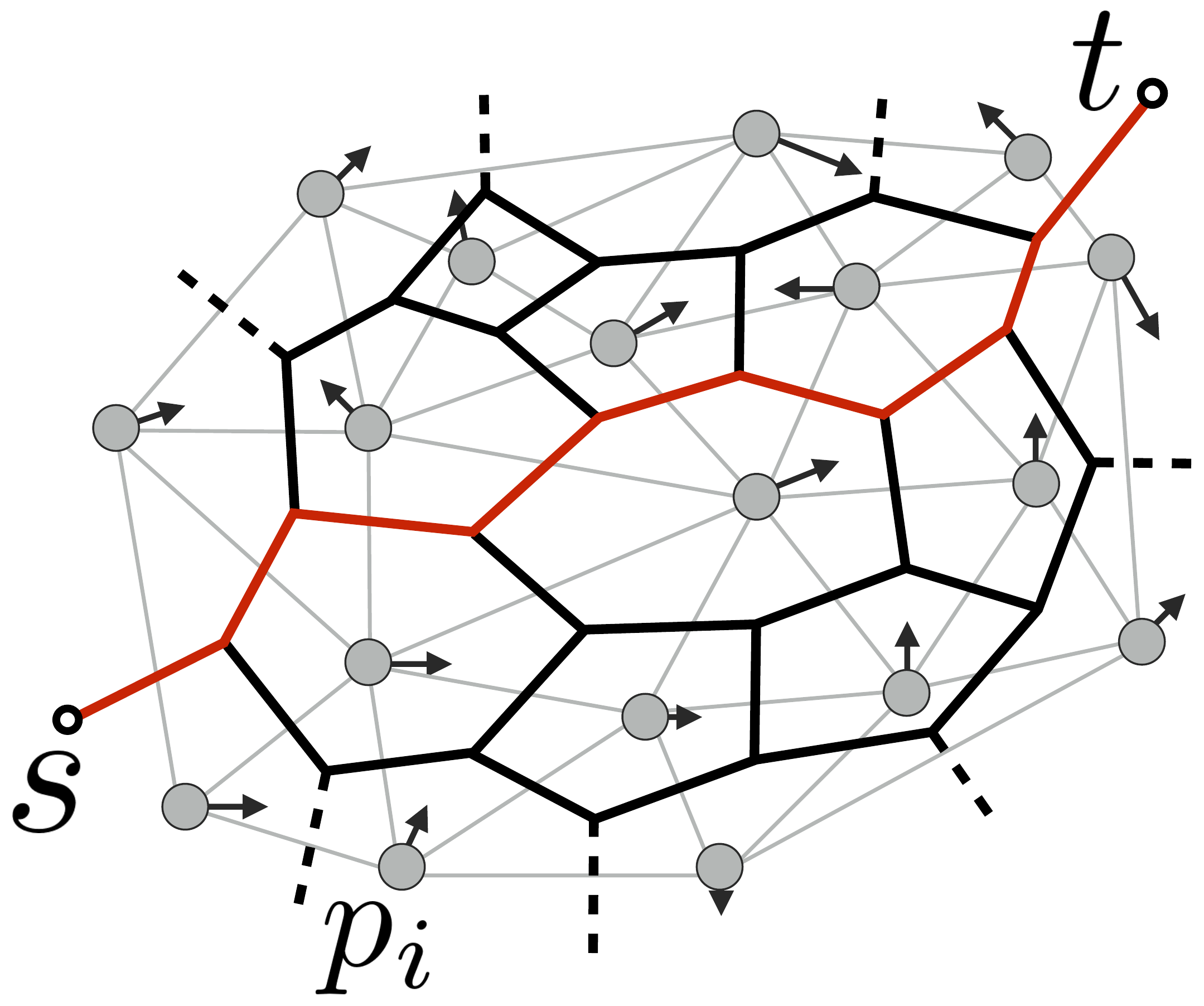}\label{fig:dual_graph_path}}
\subfloat[]{\includegraphics[width=0.25\linewidth]{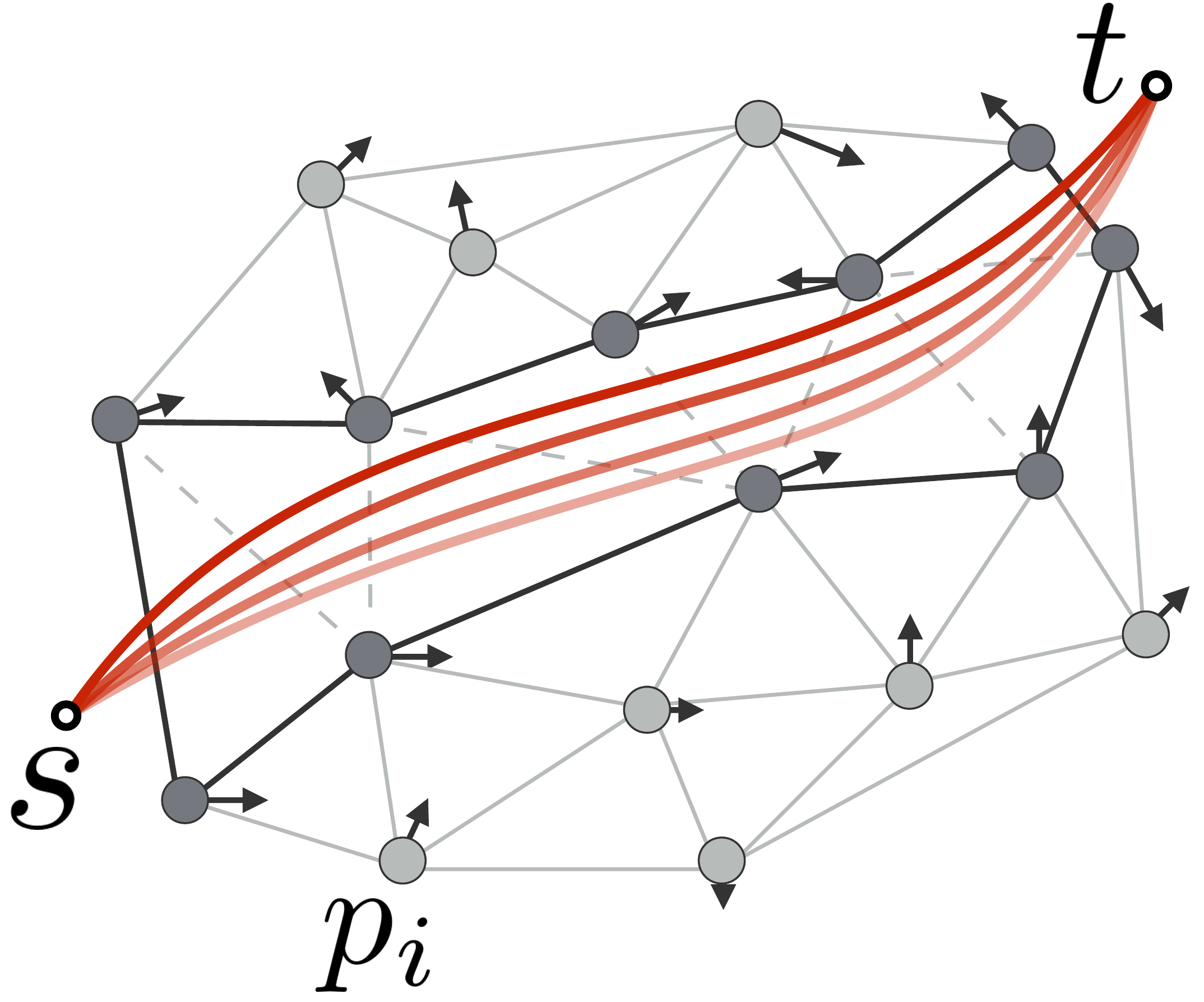}\label{fig:channel}}
\caption{(a) A path from $s$ to $t$ planned through the crowd. Pedestrians are shown as grey circles with arrows denoting velocities. (b) The path sequentially passes through pairs of pedestrians shown as darker circles connected by dashed line segments. (c) The dual graph $\mathcal{T}^*$ of the triangulation of $\mathcal{P}$. A path connecting $s$ and $t$ is shown in red. (d) A "channel" (bold line segments) and the corresponding homotopy class of paths (red curves) bounded within. } 
\label{fig:paths}
\vspace{-15pt}
\end{figure*}
An important class of planning methods leverages computational geometry to gain efficiency; Voronoi diagrams and Delaunay Triangulation have long been used for environment abstraction. With Voronoi diagrams, the decomposition of free space based on Voronoi cells efficiently generates homotopically distinct paths~\cite{Choset2000,Garber2004,Banerjee2013,Kuderer2014}. Delaunay Triangulation has also been used in graphics to solve the virtual character navigation~\cite{Kallmann2005,Broz2014}.  These data structures have important complexity properites in 2D planar environments: compared to grid-based maps, computational geometric approaches search a much smaller graph while faithfully capturing topological properties like connectivity. In~\cite{Demyen2007} Triangulation A* (TA*) and Triangulation Reduction A* (TRA*) is introduced for pathfinding in the triangulation space. Similarly,~\cite{Yan2008} presents a path planning approach in Constrained Delaunay Triangulation (CDT) to account for static polygon obstacles. Finally~\cite{Chen2010} uses dynamic Delaunay triangulation (DDT) to solve the problem navigation in the absence of a pre-specified goal position.

However, these algorithms do not reason about obstacle dynamics, and naively adding a time dimension results in exponential computation.  Our approach introduces a novel approach to incorporate obstacle dynamics in the global planning step and extends the pathfinding algorithm to efficiently analyze environmental evolution.
\vspace{-5pt}
\subsection{Human Aware Navigation}
\emph{Human aware navigation}~\cite{crowd-nav-survey} focuses on  socially acceptable robot behaviors rather than explicitly solving the collision avoidance problem. In~\cite{Trautman2010} cooperative collision avoidance is leveraged to develop a tractable interacting Gaussian Process (IGP) model. A socially-inspired crowd navigation approach is developed in~\cite{Muller2008} where the robot follows pedestrians moving towards its goal. However, this method might fail in situations where the major flow of pedestrians is not in the robot's goal direction. Other work leverages learning to model and replicate socially compliant behaviors for crowd navigation~\cite{kudereriros2013,learning-with-braids,learning-social-robot,irlnavigate,crowd-nav-deep-learning-mit}

In this work, we do not explicitly address socially acceptable navigation. However, our framework can be extended to incorporate social norms. For example, in the graph construction step (described below), the interaction between people can be encoded in the graph is obtainable (one person taking pictures for the other, a group of friends walking together) so that edges between pedestrians can be denoted as non-crossable for the robot.  Moreover, the path generated by our planner naturally captures some social norms. For example, when passing a pedestrian, our planner will favor passing from behind rather than in front of the pedestrian, because the former leads to a shorter traveling distance.

\section{Methodology}
\noindent Let the crowd positions at time $\tau$ be denoted as $\calP = \{\position \mid \position\in\reals^2, i=1,\ldots,\nt  \}$, where $\ntau$ is the number of pedestrians.  Let $\calV = \{\velocity \mid \velocity\in\reals^2, i=1,\ldots,\nt  \}$ be pedestrian velocities at $\tau$.  Let the robot starting point be $s \in \reals^2$ and goal $t\in\reals^2$. We seek the collision free trajectory that is optimal with respect to travel time and distance.
\subsection{Delaunay Triangulation}
Our first objective is to plan a path through the obstacles $\calP$ at the topological level. This is achieved by Delaunay triangulation on $\calP$, which produces the graph $\calT$ (light gray graph in Fig.~\ref{fig:dual_graph_path}; $\position$ are vertices).  The dual graph $\calTstar$ (black graph in Fig.~\ref{fig:dual_graph_path}) has one vertex for each face of $\calT$ and one edge for face pairs separated by an edge in $\calT$.  If we define homotopic equivalence as paths that can be continuously deformed into each other without passing through vertices $\position$, then $\reals^2$-paths and $\calTstar$-graphs are equivalent (Fig~\ref{fig:dual_graph_path},~\ref{fig:channel}):
\begin{theorem}
  Any path through $\calP$ in $\reals^2$ uniquely determines a path in $\calTstar$.
\end{theorem}
\begin{theorem}\label{theorem:channel_to_homotopy}
A path on $\calTstar$ without repeating vertices corresponds to one homotopy class of paths in $\reals^2$
\end{theorem}

For brevity, we prove these theorems in a supplement.
\footnote{\href{http://caochao.me/files/proof.pdf}{http://caochao.me/files/proof.pdf}} \emph{Thus, instead of optimizing paths in $\reals^2$ we can equivalently search the graph $\calTstar$.  This graph search is computationally efficient  (Section~\ref{sec:experimental-comparison})}.  An important advantage of searching $\calTstar$ is that we can reason about pedestrian pair interactions, instead of individually checking collisions. Consider a path passing through a group of people as shown in Fig. \ref{fig:path}, and the directional curve showing the robot path. By projecting pedestrian velocities onto triangulation edges, triangulation time evolution is achieved. In this manner, pedestrian dynamics are analyzed as a network.

\subsection{Dynamic Channel}
A path on $\calTstar$ (Fig.~\ref{fig:dual_graph_path}) is equivalent to a triangulated simple polygon called a ``channel''~\cite{Kallmann2005} (Fig.~\ref{fig:channel}). We extend this concept to a ``dynamic channel.''   This dynamic channel is precisely the time evolution of our triangulation.

Pedestrian dynamics cause the dynamic channel to deform constantly. Of most concern is ``gate'' change (light gray dashed lines in Fig.~\ref{fig:channel}), the distance between pedestrian $i$ and $j$ at time $\tau$:
\begin{equation}\label{eq:1}
D^{ij}(\tau) = ||p^i_\tau - p^j_\tau||
\end{equation}
where $||.||$ denotes the Euclidean distance. Let $\dthresh = 2(r_{obs} + S_{safe})$ be a threshold for the width of a gate considered feasible for the robot to pass, where $r_{obs}$ is the radius of a pedestrian expanded by the radius of the robot, and $S_{safe}$ is the minimum safe distance between the robot and pedestrians. Solving for
\begin{equation}\label{eq:2}
D^{ij}(\tau) \geq \dthresh
\end{equation}
we obtain the time intervals $T_{feasible} = [\tau_1, \tau_2] \cup \dots\cup [\tau_{m-1}, \tau_{m}]$ when the gate is wide enough for the robot to cross. Given the Estimated Time of Arrival (ETA) for the robot to reach the gate $\tau_{ETA}$ (detailed in the next subsection), safe passage is guaranteed when $\tau_{ETA} \in T_{feasible}$. This condition must hold for all gates to guarantee safe passage.

We assume that pedestrian behaviors can be modeled linearly. Thus the trajectory (starting from the current frame $\tau=0$) of pedestrian $i$ can be described by:
\begin{equation}
 p^i(t) = p_0^i + v^i\tau
\end{equation}
where $p_0$ and $v^i$ are the current position and velocity. $D^{ij}(\tau)$ is then a quadratic function with respect to $\tau$ and the graph of $D^{ij}(\tau)$ is a parabola opening upward as in Fig. \ref{fig:curve}. Solving for Equation \ref{eq:2} gives zero, one or two solutions. When there is no solution, the minimum distance $D^{ij}_{min}(\tau)$ between pedestrians $i$ and $j$ will never be too close for the robot to cross. When there is only one solution, a critical moment exists when the robot can be just safe enough to cross. When there are two solutions, denoting $t_1$ and $t_2$ with $t_1 < t_2$, it is required that $t_{ETA} \in ([-\infty, t_1] \cup [t_2, +\infty]) \cap [0, +\infty]$ for safe passage as shown in Fig. \ref{fig:curve}.

\begin{figure}[b]
\vspace{-5pt}
\centering
\subfloat[]{\includegraphics[width=0.47\linewidth]{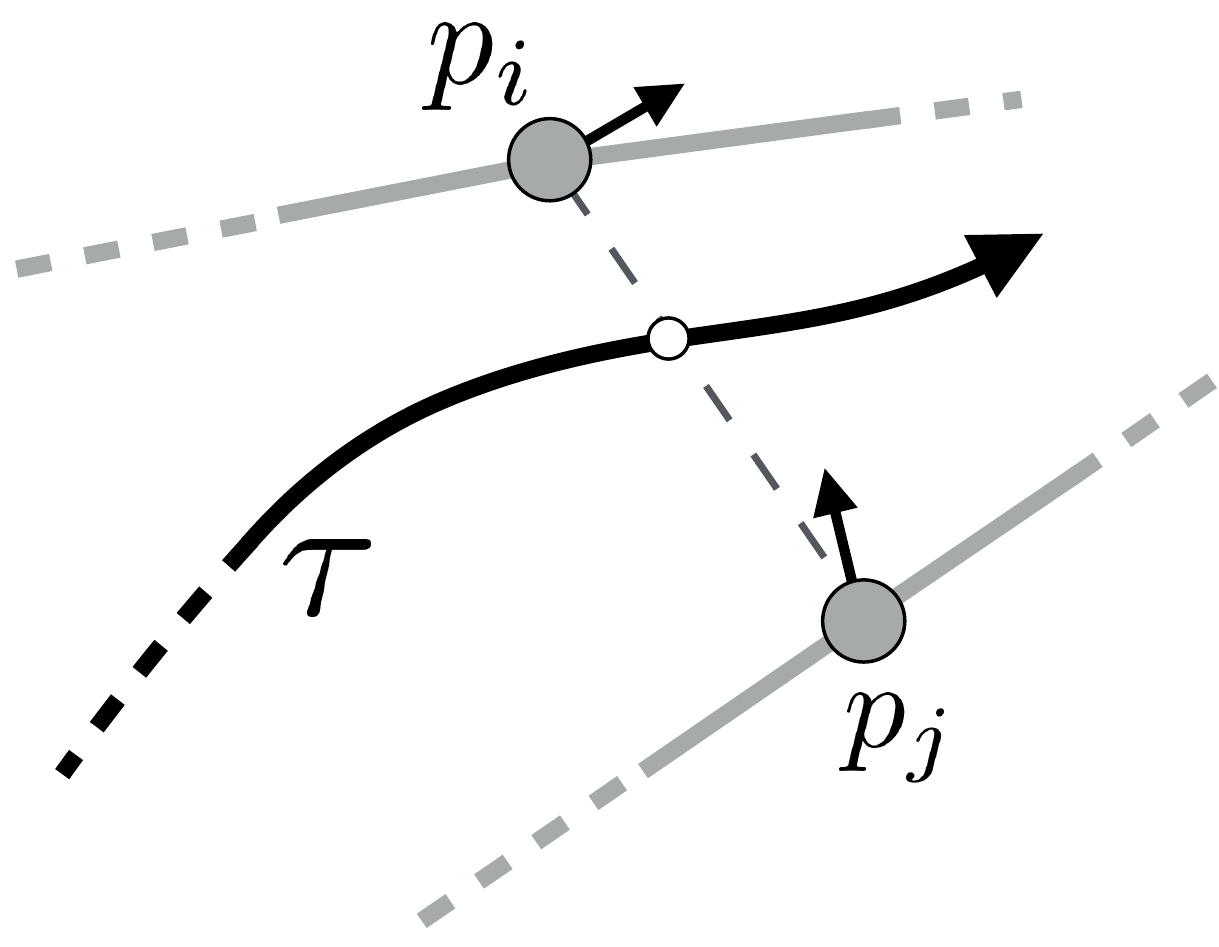}\label{fig:gate}}
~
\subfloat[]{\includegraphics[width=0.47\linewidth]{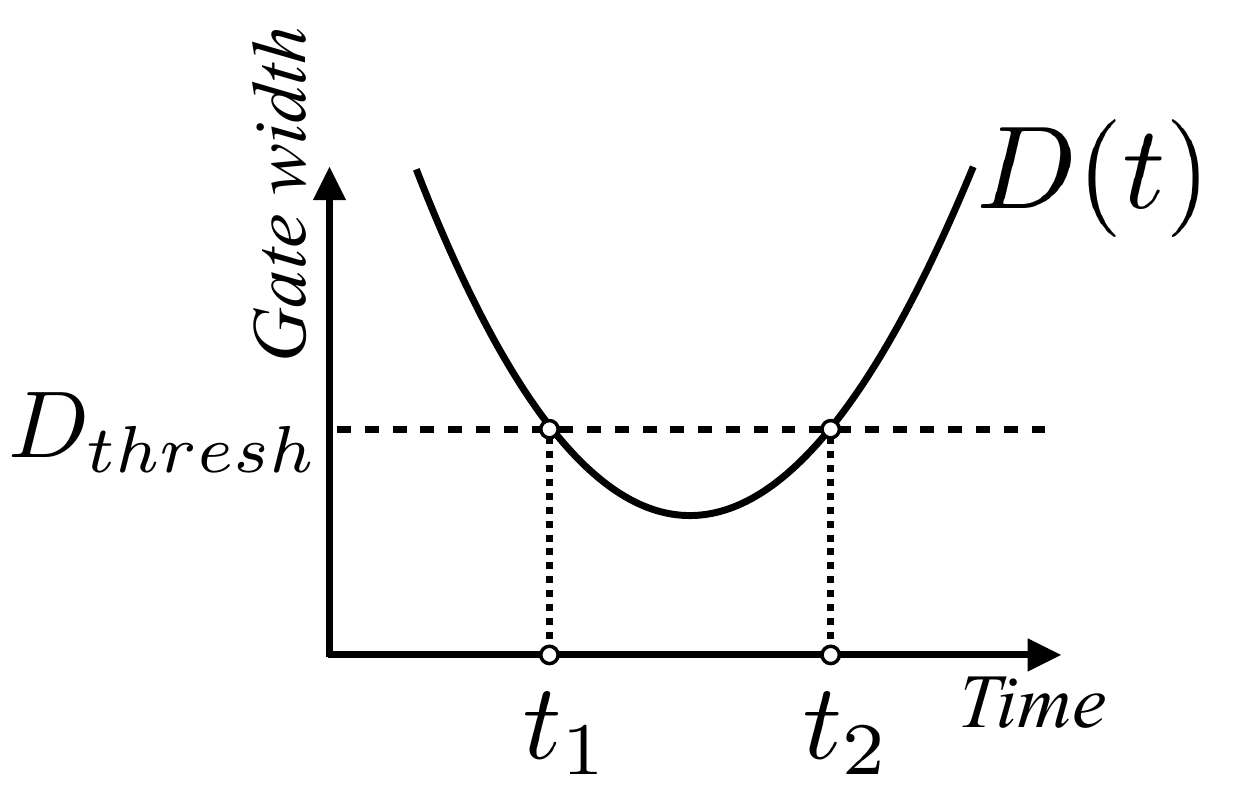}\label{fig:curve}}
\caption{(a) A gate formed by pedestrians $p_i$ and $p_j$. A path $\tau$ is planned through the gate. (b) The distance between two pedestrians changes with time as a parabola. $t_1$ and $t_2$ show the moments when the gate is  wide enough (above the threshold $\dthresh$) for safe passage.}
\label{fig:gate_dynamic}
\vspace{-0pt}
\end{figure}

\subsection{Timed A* Search}
A* algorithm has been used for searching for shortest paths in the triangulation space~\cite{Kallmann2005,Demyen2007}. Care needs to be taken for computing the \textit{g} and \textit{h} values. \cite{Yan2008} proposes the "target attractive principle" to determine the placements of nodes when building the dual graph on constrained Delaunay Triangulation. Here we adopt the rule-based method as in \cite{Yan2008} to determine the graph nodes. Note that in our work, the constraints are introduced by non-crossable edges that will lead to potential collision rather than static obstacles. For compactness, the detailed discussion on node placement is not given here.

Here, we extend the A* algorithm by incorporating robot and pedestrian dynamics. Within the triangulation framework, we focus on the interaction between connected pedestrian pairs. Thus the dynamics of the robot and pedestrians can be studied by projecting them along the edges in the triangulation. In particular, we compute the \textit{g} and \textit{h}-value with a lookahead time period. While searching, the velocity of a node representing a triangle can be determined by interpolating the velocities of three endpoints, which gives $g(\tau), h(\tau)$.  We then compute $\tau_{ETA}$ to determine $g(\tau_{ETA})$ and $h(\tau_{ETA})$. Similar to the estimate of \textit{g} value in a regular A* algorithm, we track path history while searching. For each node, the current path to the node is time parameterized by taking into account the kinematics and dynamics of the robot and curvature of the path. Note that due to the approximated node placements at searching time as in \cite{Yan2008} and the linear model assumption, $\tau_{ETA}$ can only be approximated accurately in a finite time horizon. Due to the time-variant $g(\tau)$ and $h(\tau)$ value, our planner will seek a candidate shortest path. For instance, when passing a walking pedestrian, our planner will passing behind rather than in front of, even though in front of is currently a shorter path.

\subsection{Path Optimization}
The output from our Timed A* algorithm is a path on $\calTstar$ consisting of a series of triangles that form a channel. We find the shortest path in the channel that ensures sufficient clearance. \textit{Funnel algorithm} \cite{Chazelle1982}, \cite{Lee1984}, \cite{Hershberger1991} can be used for computing a shortest path within a simple triangulated polygon connecting the start and goal points.  To account for non-zero radius obstacles, \cite{Demyen2007} proposed a modified \textit{funnel algorithm} by inserting conceptual circles at each vertex of the channel. The resulting path is composed of alternating straight-line segments and arcs. In our work, we further extend this method to account for the pedestrian dynamics. Given a pedestrian $\position$, we project the velocity $\velocity$ onto the vector $e_{ji} = p_j - p_i$, where $p_j$ is the other end of the edge in the channel. If the projection $v'_i > 0$, the path will pass by $\position$ in the front, potentially blocking $\position$'s moving direction. In this case, we assign a circle with the radius $r = k||v'_i||$ around $\position$ to create clearance, where $k$ is a tunable parameter. Similarly, if $v'_i < 0$, the path will pass from behind. Figure~\ref{fig:funnel} provides an illustration, where pedestrians $p_i$, $p_j$ and $p_k$ are assigned different sized circles based on their velocities.

\begin{figure}[!t]
\vspace{8pt}
\centering
\subfloat[]{\includegraphics[width=1.0\linewidth]{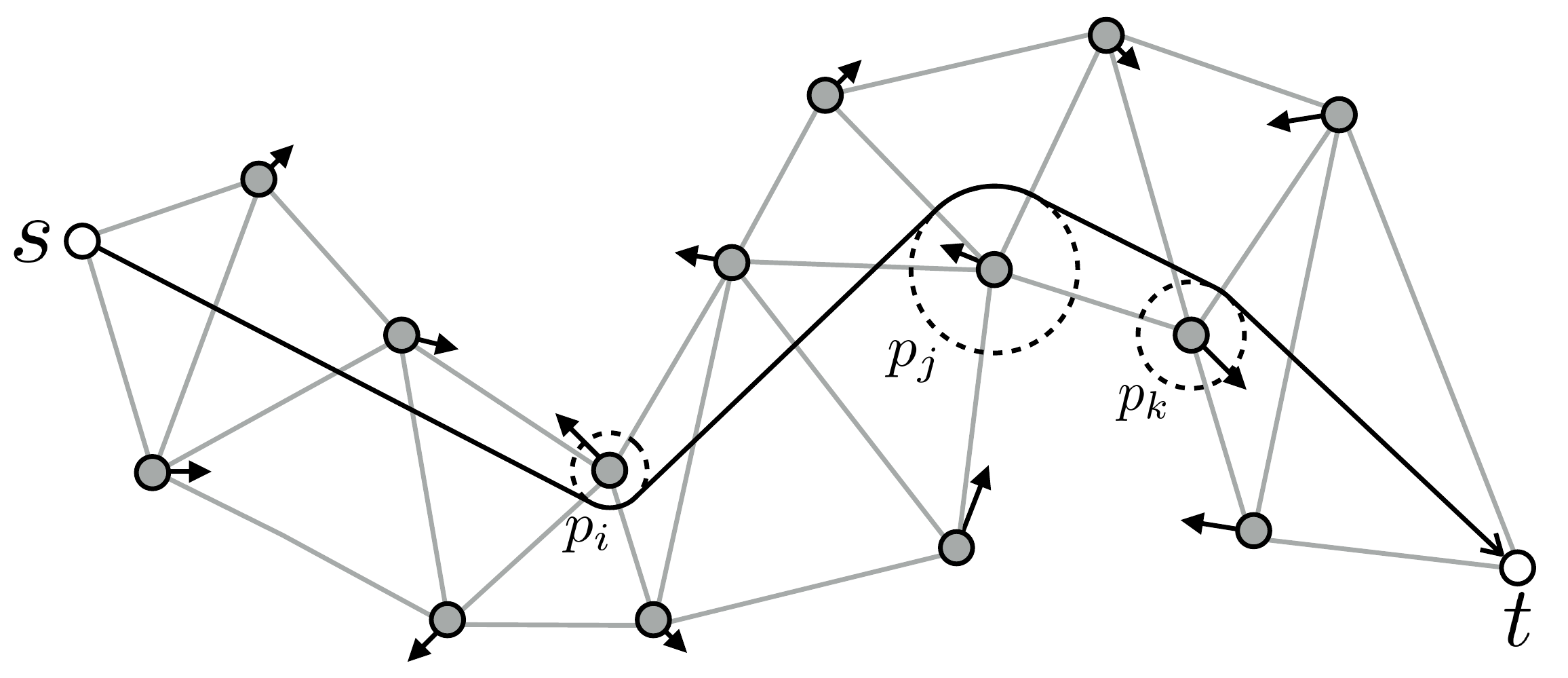}\label{fig:funnel}}
\caption{Within a channel the shortest path is determined by our  \textit{funnel algorithm}. Clearance radius is determined by whether the pedestrian is "threatening" or "enlarging" the channel. Larger radius is assigned for the former ($p_j$) while smaller radius is assigned to the latter ($p_i, p_k$).}
\label{fig:algo}
\vspace{-0pt}
\end{figure}

Most of the collision avoidance effort described above is implicitly done by enforcing timely traversal. However, when obstacles move faster than the robot simply following the path cannot guarantee  collision-free navigation. In such cases, a low-level robot-specific controller is needed to perform maneuvers for the robot to avoid collision. In a Delaunay Triangulation, the circumcircle of any triangle will not include other points inside. Thanks to this property, a collision-free circular area can be identified in the robot's vicinity, which can be used for local maneuvers when needed. Furthermore, if no paths are found by the Timed A* algorithm, other behavioral strategies can be employed to utilizes the local collision-free area. For example, the robot can follow the people walking in the goal direction while maintaining a proper clearance to others as in \cite{Muller2008}.

\section{Experimental Evaluation}
\noindent We demonstrate the advantage of our planner by testing on a large number of navigation tasks in simulated environments based on public human-trajectory datasets. We compare success rate, efficiency of traversal, and time complexity of our algorithm against two other planning methods. Finally, we empirically study the effect of observability on our algorithm.

\subsection{Experiment Settings}
\subsubsection{Datasets}
For experiments, we implemented our planning framework to control a car-like robot navigating in simulated environments. In particular, we use the pure-pursuit algorithm \cite{coulter1992implementation} to track a path output from our planner. For benchmarking, in our simulation we replay the human-trajectories from two public datasets: ETH \cite{walkalone} and UCY \cite{Lerner2007}. There are five subsets in total recorded in different scenarios, denoted as $eth\_hotel$, $eth\_univ$, $ucy\_univ$, $ucy\_zara01$ and $ucy\_zara02$.

For each dataset, the trajectory data are first interpolated and then re-projected back to the world coordinate using the homography matrices provided. We define a rectangular region as the workspace which is bounded by extreme values of the pedestrian positions of each dataset. For the navigation task, the robot starts at four mid-points of the workspace boundaries and moves to the antipodal goal position as shown in fig \ref{fig:experiment_setting}. In the experiments, only pedestrians are considered to be obstacles and all free space is assumed traversable for the robot. Fig. \ref{fig:eth_extracted} shows an example testing configuration, where pedestrians are shown as yellow circles with black arrows indicating velocity. The starting positions are shown as red circles with arrows. The velocities of pedestrians are approximated by the position differences between two consecutive frames divided by the sampling period.

Trials begin at a sequence of starting times at intervals of $3s$, which results in $1520$, $1148$, $184$, $148$, and $120$ trials for each dataset respectively. All experiments were run on a  2.6 GHz CPU/15.5 GB memory computer.

\begin{figure}[!t]
\centering
\subfloat[]{\includegraphics[height=4.4cm]{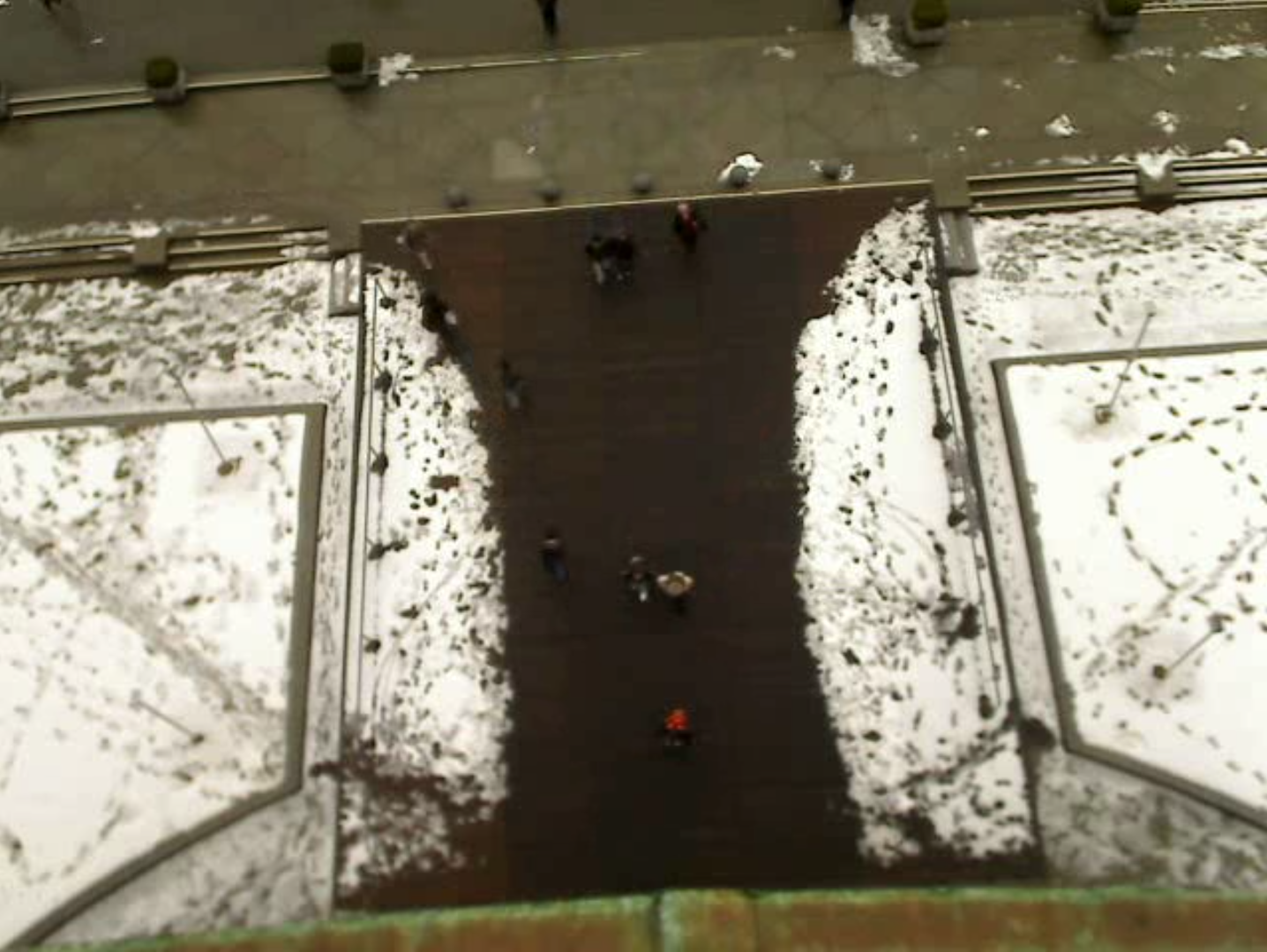}\label{fig:eth_real}}
\subfloat[]{\includegraphics[height=4.4cm]{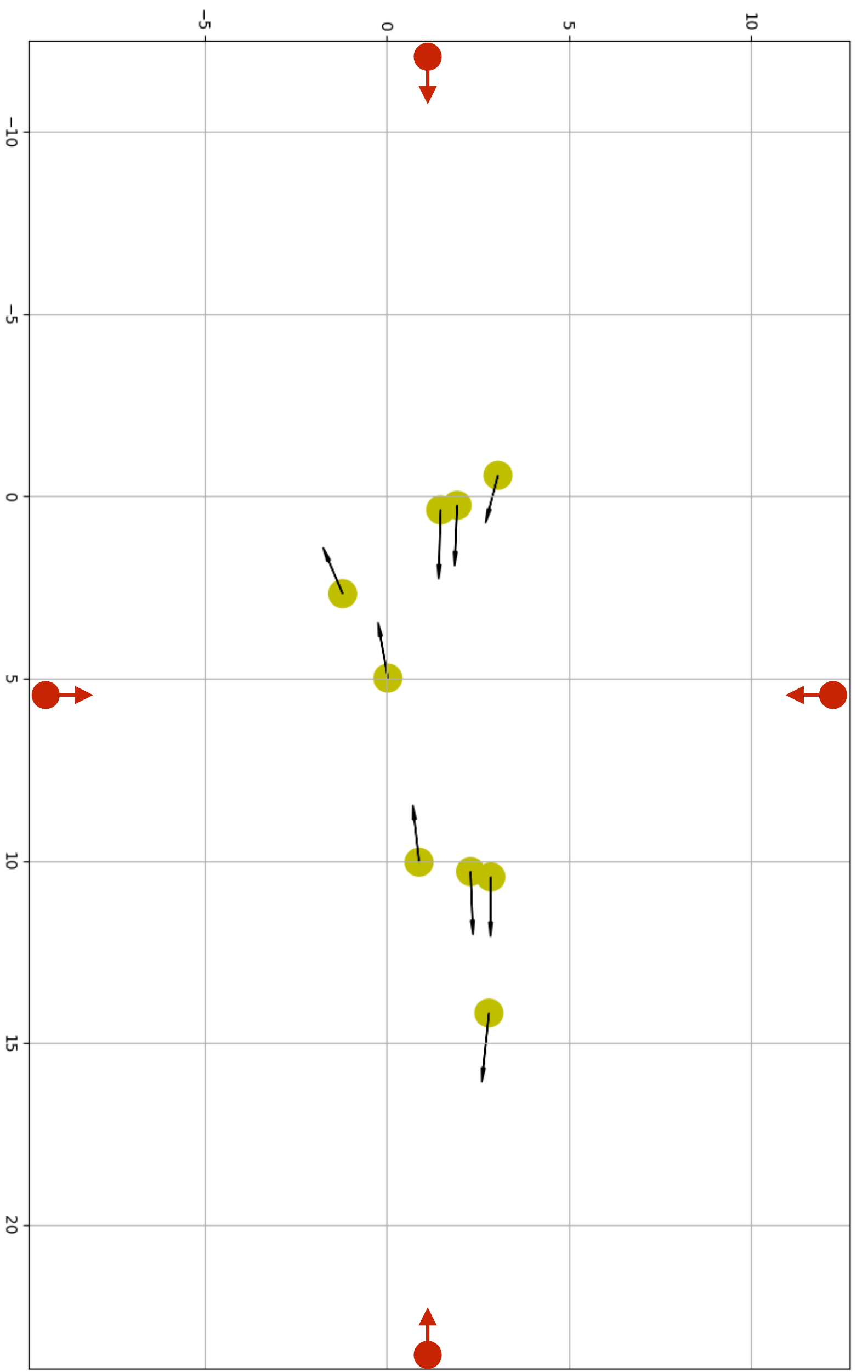}\label{fig:eth_extracted}}
\caption{Experiment Settings. (a) shows a frame from the ETH dataset. (b) shows the pedestrian positions extracted from the frame. Red circles with arrows show the starting positions of the robot in the experiments.}
\label{fig:experiment_setting}
\vspace{-0pt}
\end{figure}

\subsubsection{Generalized Velocity Obstacle Planner}
For comparison, we implemented a planning algorithm based on the Generalized Velocity Obstacle (GVO,~\cite{Wilkie2009}) to control the same car-like robot. The algorithm first randomly samples a control input $u$ from a feasible set. Then a trajectory is simulated within a time horizon for the input, which takes into account the robot dynamics. The trajectory is then checked for collision with  GVO prediction on each moving agent. Finally, among all collision-free control inputs, the one closest to the \textit{preferred} control $u^*$ is chosen. As suggested in the paper, we choose the \textit{preferred} control to be the one driving the robot directly to the goal as if there is no obstacles. However, we do not explicitly solve the optimization problem for $t^*$ that results in minimum distance between the robot trajectory and obstacles. Instead, we discretize the trajectory and check collision at 10Hz. Only in this manner can we run the algorithm efficiently enough for the benchmarking. The other parameters used were $time\_horizon=3.5s$, $wheelbase = 1m$ and $sampling\_per\_timestep = 40$.

\subsubsection{"Wait-and-go" Planner}
We designed the simplest possible crowd navigation strategy, ``wait-and-go'', as a baseline: the robot drives in a straight-line towards the goal. When the robot comes too close to pedestrians or there is a potential collision determined by  Velocity Obstacle, the robot stops (wait) and resumes moving (go) when possible. 

\subsection{Comparing success rate, efficiency and time complexity}
\label{sec:experimental-comparison}
For each trial, the first frame is concatenated to the last one to form a cycle. We run each trial until the robot reaches the goal position or all the frames have run out. If the distance between the robot and any pedestrian is less than $1m$, a collision is reported and the trial is failed. We also calculated running time to compare efficiency of the three approaches. Note that when a trial fails, the timer will not stop until the trial is finished. The robot has a speed limit of $1.2m/s$ which is slightly slower than the average human walking speed.

\begin{figure}[!t]
	\centering
	\subfloat[Success Rate]{\includegraphics[width=0.4\textwidth]{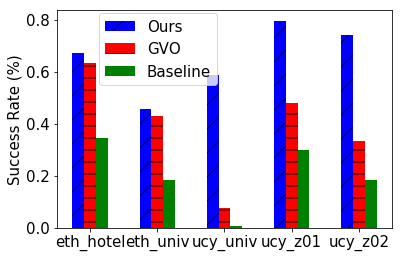}\label{fig:success_rate}}\\
	\subfloat[Efficiency]{\includegraphics[width=0.4\textwidth]{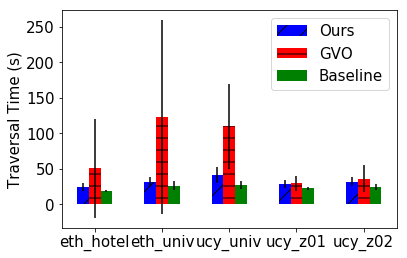}\label{fig:efficiency}}
	\caption{Experiment Results. Comparison of: (a) the success rate of navigation tasks in each dataset. (b) The average traveling time of the robot in each dataset. Vertical line segments show the variance by a standard deviation.}
	\label{fig:results_success_rate}
	\vspace{-16pt}
\end{figure}

\begin{figure}[!t]
	\centering
	\subfloat[Time Complexity]{\includegraphics[width=0.4\textwidth]{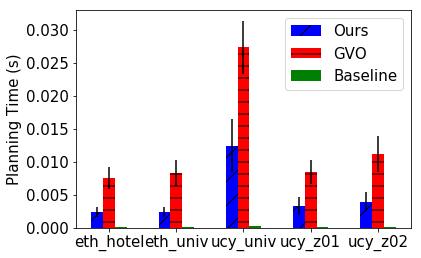}\label{fig:time_complexity}}\\
	\subfloat[Number of Pedestrians]{\includegraphics[width=0.4\textwidth]{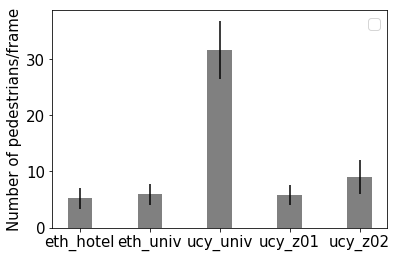}\label{fig:crowd_density}}
	\caption{Experiment Results. Comparison of: (a) the average planning time of a frame for each dataset. (b) The average number of pedestrians in each dataset. Vertical line segments show the variance by a standard deviation.}
	\label{fig:results_success_rate}
	\vspace{-15pt}
\end{figure}

\subsubsection{Success Rate}
From the result shown in Fig. \ref{fig:success_rate}, our planner has the highest success rates in all five datasets, while GVO ranks  second and the baseline planner performs last. For the challenging dataset \textit{ucy\_univ}, where the crowd density reaches $30$ pedestrians per frame, our planner significantly outperforms the baselines. Anecdotally, most of the failure cases are due to sudden pedestrian appearance and nonlinear pedestrian movement.

\subsubsection{Traveling Efficiency}
For traversal efficiency, as shown in Fig. \ref{fig:efficiency}, our planner is more efficient than "wait-and-go," while GVO exhibits much longer times with much larger variance. This is because GVO tends to passively dodge approaching pedestrians (resulting more evasive strategies), while our planner actively looks for open spaces and a long-term shortest route.

\subsubsection{Computational Complexity}
For computational complexity (~\ref{fig:time_complexity}), both our approach and GVO sees an increase in computation as more pedestrians enter the scene. However, our method is significantly more time efficient than GVO.

\begin{figure}[!t]
\vspace{0pt}
\centering
\subfloat[$t = 2s$]{\includegraphics[width=0.243\textwidth]{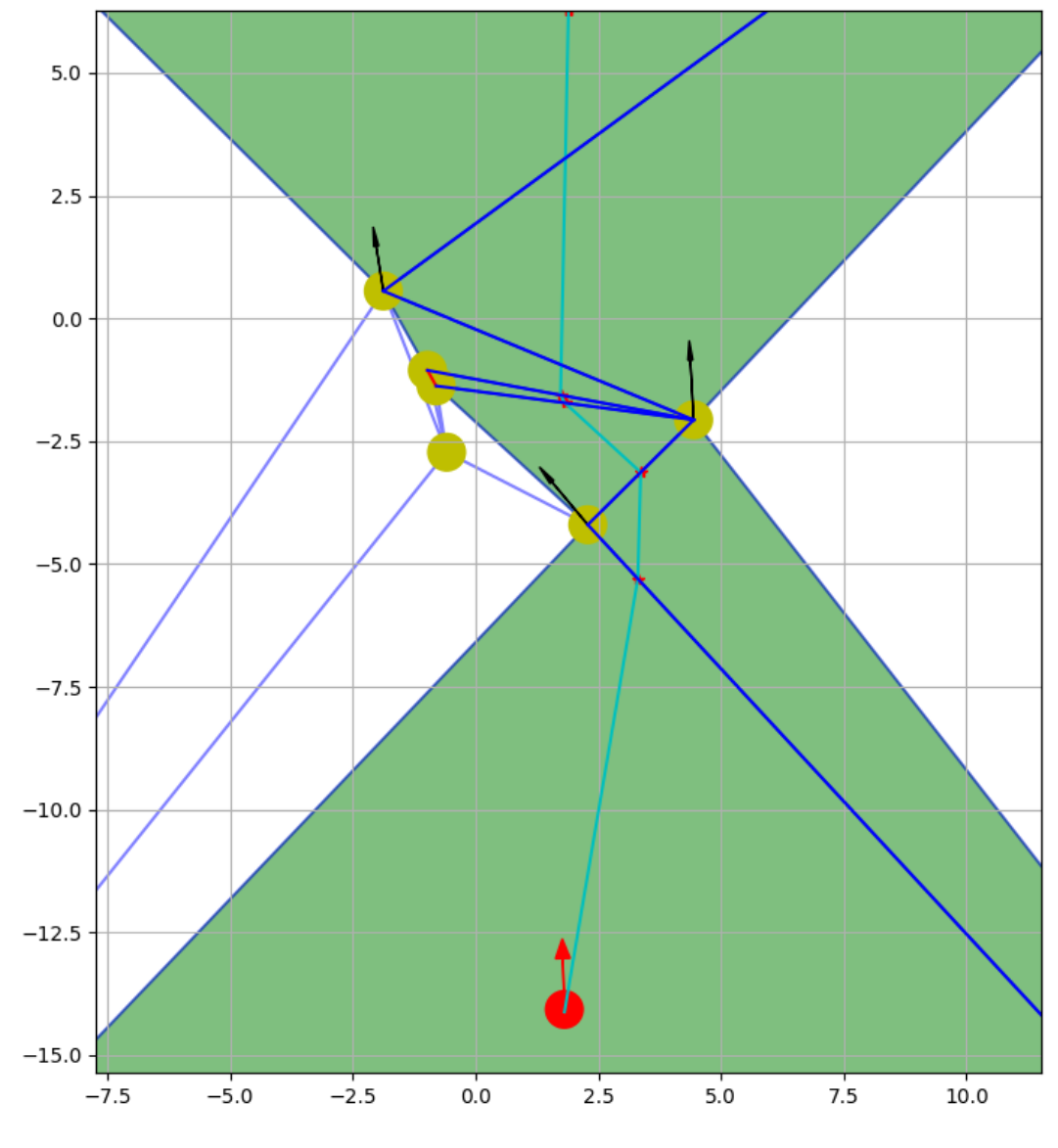}}
\subfloat[$t = 7.25s$]{\includegraphics[width=0.243\textwidth]{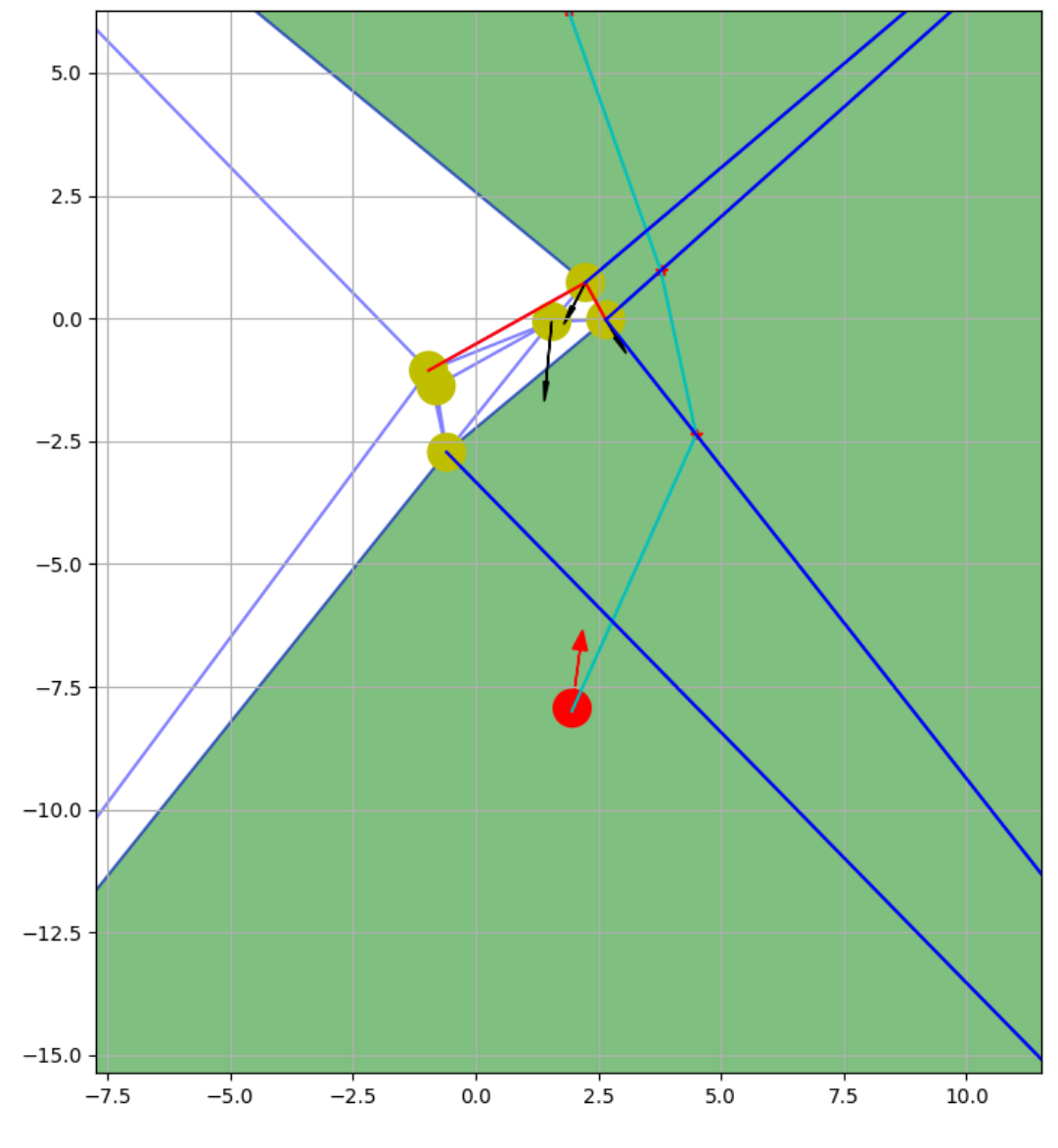}}\\
\subfloat[$t = 11.25s$]{\includegraphics[width=0.243\textwidth]{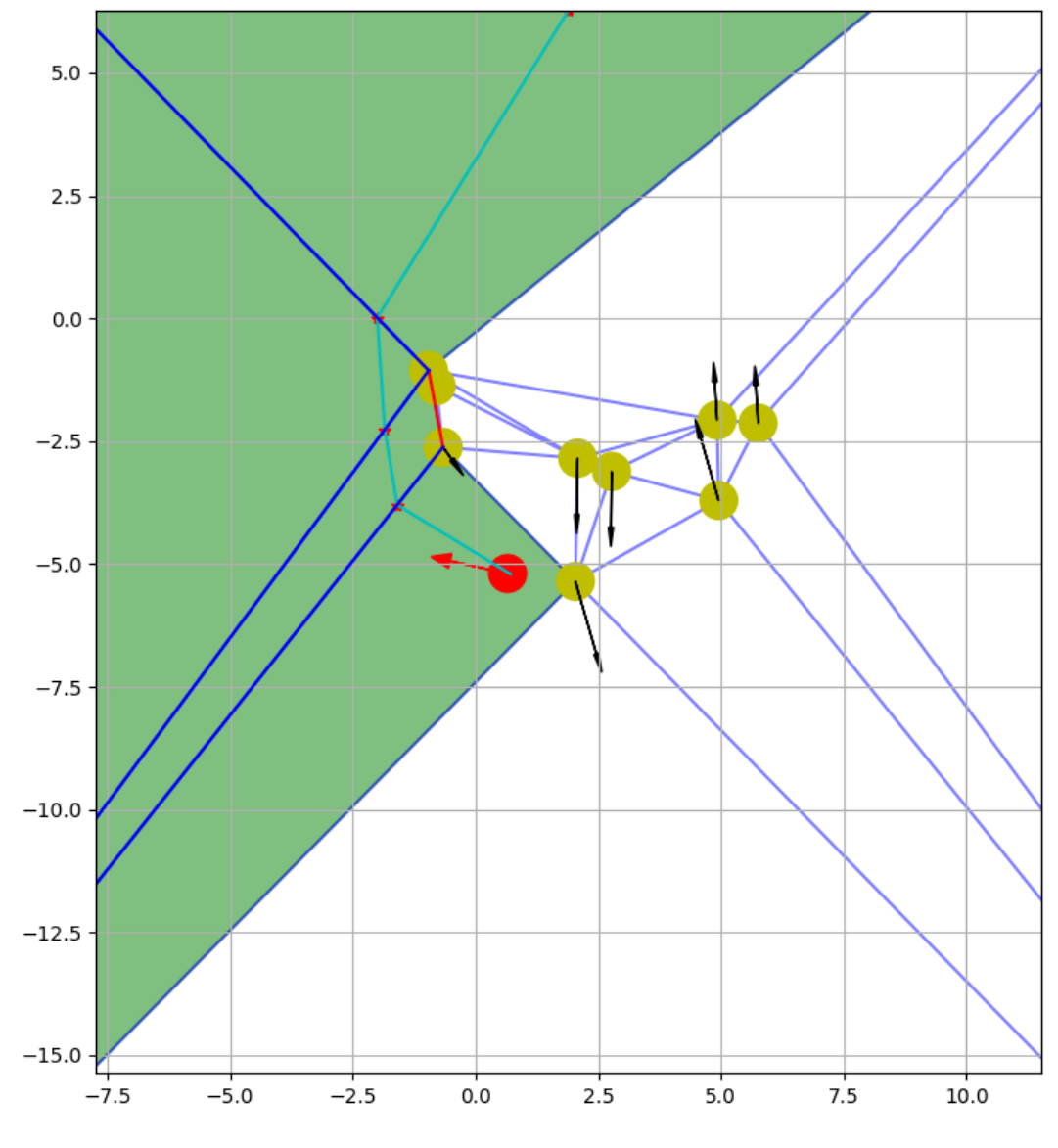}}
\subfloat[$t = 15s$]{\includegraphics[width=0.243\textwidth]{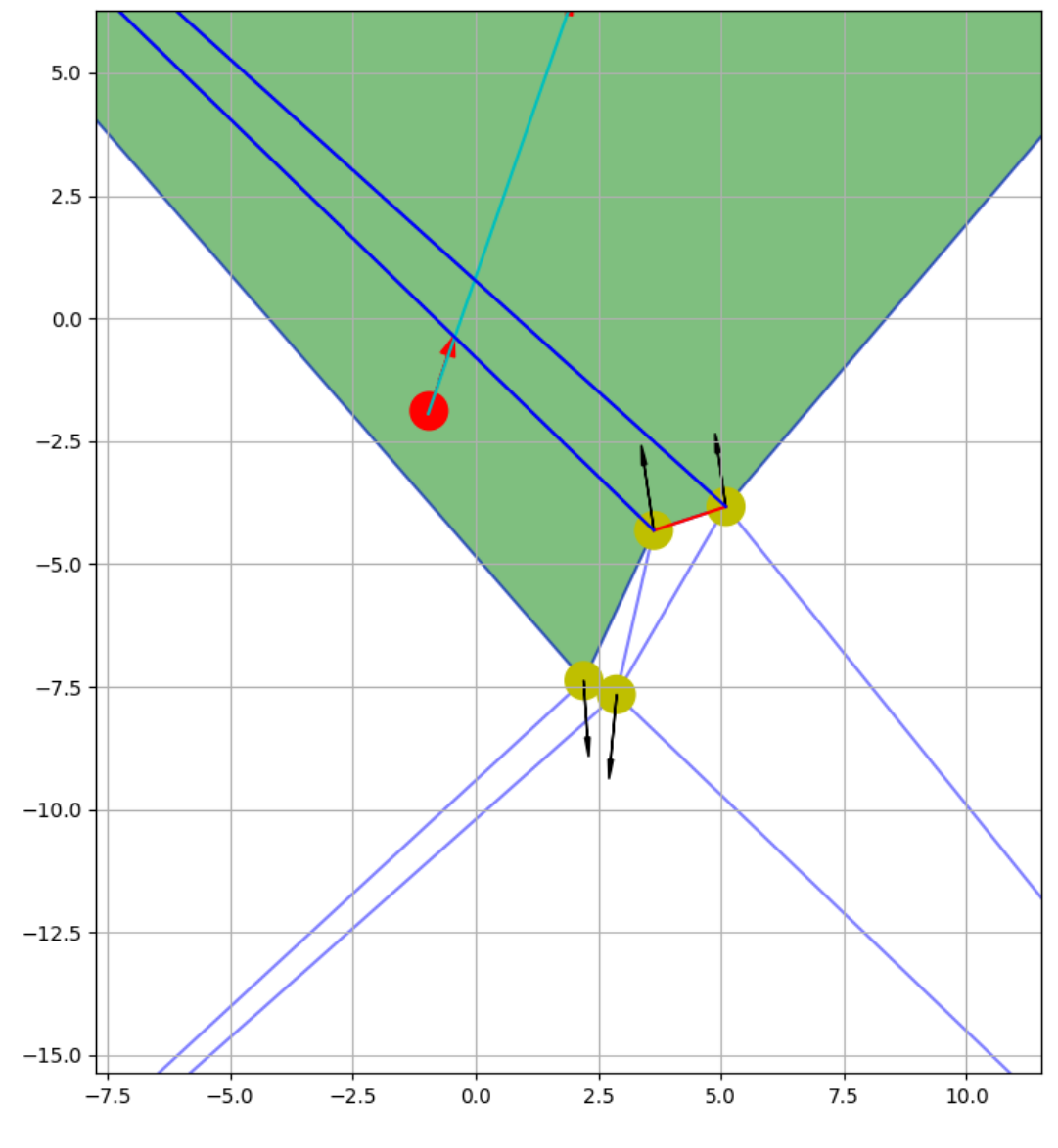}}
\caption{The path found by our planner in the \textit{eth\_univ} dataset. Green areas show dynamic channel. Cyan line shows  trajectory. Blue line segments show valid gates to cross; red ones show unsafe gates. The robot (red filled circle) attempts to navigate from bottom to  top. Four static pedestrians inhabit the four corners of the workspace, allowing robot to travel to the crowd's side.}
\label{fig:sequence}
\vspace{-0pt}
\end{figure}

Qualitatively, Fig. \ref{fig:sequence} shows a sequence of the robot navigating in the $eth\_hotel$ dataset. At the beginning, the robot plans to pass by the pedestrians from the right. Between $t = 7.25$ and $t = 11.25$, there suddenly appears a group of people on the right side. Our planner is able to adapt to this and steer the robot to the left for a safe-pass.


\subsection{Observability Effects}
We also study observability effects by testing our planner on a larger simulated scene with limited sensor ranges. The simulated pedestrians are generated as follows. For every $5s$, $2$ pedestrians are spawned from the two sides of the workspace. Each pedestrian is assigned with a random goal position at the other side and a random linear velocity towards the goal. While traveling to the goal the pedestrians will randomly change their goal and speeds every $3s$.  The observation range of the robot is tested from $1m$ to $30m$, with the interval being $1m$ when under $10m$ and $5m$ when above $10m$. We also run trials without sensor range limit for comparison. Beyond the sensing range, only the goal position is known, so the robot has no information about  pedestrians outside the field of view. In each trial, the robot navigates from the lower right corner to the upper left corner or reversely. For each sensing range, 100 trials are performed. Fig. \ref{fig:local_sequence} shows an example scene when the robot navigates in the synthetic crowd with the sensing range limited to $30m$.

\begin{figure}[t]
\vspace{11pt}
\centering
\subfloat[]{\includegraphics[width=1.0\linewidth]{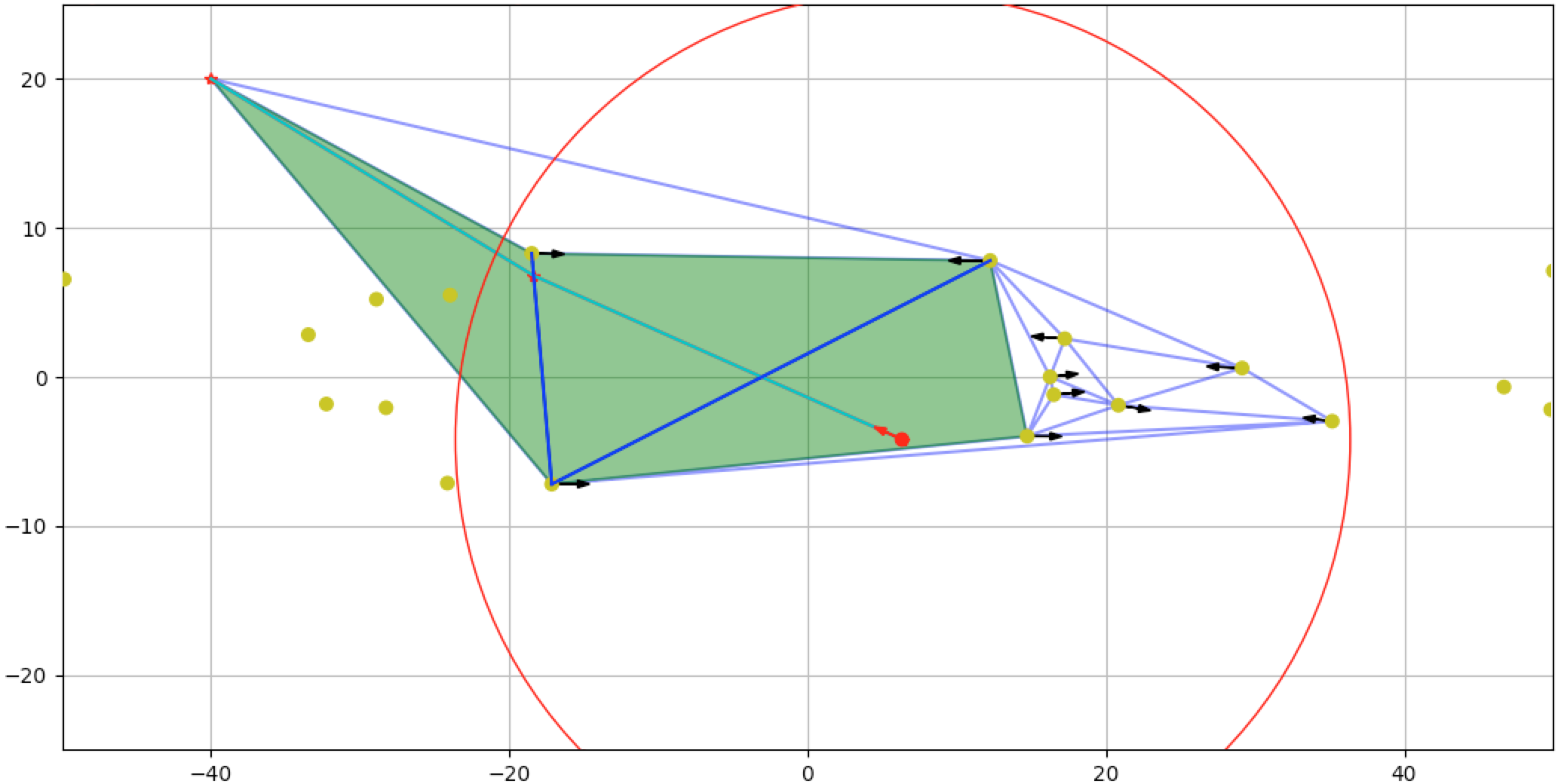}}
\caption{Path found by our planner in a synthetic dataset with a limited sensor range, denoted by the red circle. Pedestrians outside the sensor range are not included in the triangulation thus not observed by the planner.}
\label{fig:local_sequence}
\vspace{-0pt}
\end{figure}

\begin{figure}[!ht]
	\centering
	\subfloat[Success Rate]{\includegraphics[width=0.5\linewidth]{success_rate.png}\label{fig:sensor_range_success_rate}}
	\subfloat[Efficiency]{\includegraphics[width=0.5\linewidth]{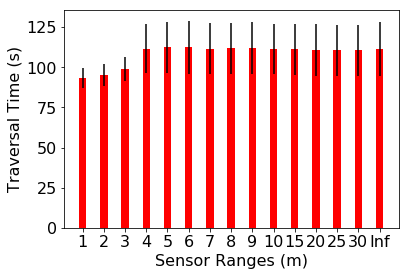}\label{fig:sensor_range_travel_time}}\\
	\subfloat[Time Complexity]{\includegraphics[width=0.515\linewidth]{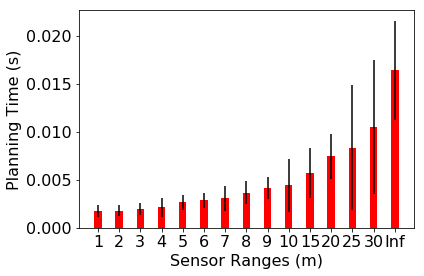}\label{fig:sensor_range_planning_time}}
	\subfloat[Number of Pedestrians]{\includegraphics[width=0.485\linewidth]{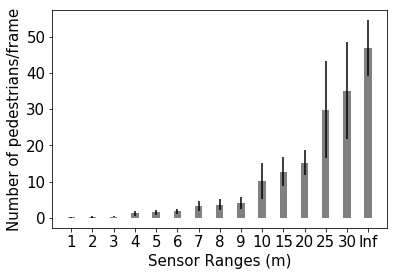}\label{fig:sensor_range_ped_num}}
	\caption{Experiment Results on the effect of observability. Comparison of: (a) average navigation success rates,  (b) average traveling  time, (c) average planning time and (d) average number of pedestrians seen by the planner with different sensing ranges. Vertical line segments show standard deviation.}
	\label{fig:sensor_range_result}
	\vspace{-0pt}
\end{figure}

Results are presented in Fig. \ref{fig:sensor_range_result}. In Fig. \ref{fig:sensor_range_success_rate} and Fig. \ref{fig:sensor_range_travel_time}, we observe that the sensor range does not make a big difference for the success rate and traveling time beyond $3m$.  Medium ranges (3 to 9m) are only slightly better than long ranges (10 to without limit). This is because the longer the channel, the less accurate the $\tau_{ETA}$ will be, and thus the guarantee for safe-passage for later portions of the channel are less likely to be valid. A practical solution to this problem is to set a time horizon for the planner. Whenever the estimated $\tau_{ETA}$ is longer than the horizon, we release the safety check. For the time complexity in Fig. \ref{fig:sensor_range_planning_time}, we see that the time needed for a planning cycle increases rapidly with the increasing number of pedestrians per frame as in Fig. \ref{fig:sensor_range_ped_num}. However, our algorithm is still efficient in such crowded environments. It can be seen from the figure that for pedestrian number over $50$, our method re-plans at $50$Hz and with less than $30$ pedestrians our method re-plans at $100$Hz.

\section{Conclusions and Future Work}
\noindent We introduced a geometry-based planning framework for crowd navigation. Our primary contribution is efficient global planning while accounting for obstacle dynamics; a proof of completeness and optimality is provided. Empirically, dynamic channels improved success rate by up to a factor of six in the most crowded environments, task completion by up to a factor of three, and decreased computational burden by a factor of two.  However, our framework has limitations. First, the method assumes perfect measurements. Second, we assume linear pedestrian motion models; even though our method was evaluated on real world pedestrian datasets with extended periods of nonlinear motion, further experiments are needed to validate robustness in interactive crowds.  Future work will integrate dynamic channels with low-level planners to increase this robustness. Additionally, we will extend our architecture for multi-robot navigation, crowd simulation, and the handling of both dynamic and static obstacles with arbitrary shapes and movement uncertainty.

\section{On the Completeness of Pathfinding in Triangulation}
\noindent We present a detailed discussion on the completeness of pathfinding in triangulation. Let the crowd positions at time $\tau$ be denoted as $\calP = \{\position \mid \position\in\reals^2, i=1,\ldots,\nt  \}$, where $\ntau$ is the number of pedestrians, the robot starting point be $s \in \mathbb{R}^2$, and the robot ending point be $t \in \mathbb{R}^2 ~ (t \neq s)$.  The task is to plan homotopically distinct paths through the crowd. Without loss of generality, we assume that $s$ and $t$ are outside of the convex hull of $\calP$.  We provide the following important definitions.

\begin{defi}[\textbf{Path}]\label{def:path}
A path is a continuous mapping $\bff : [0, 1] \rightarrow \mathbb{R}^2$ with $\bff(0) = s$ and $\bff(1) = t$.
\end{defi}
\begin{defi}[\textbf{Path homotopy with respect to $\calP$}]\label{def:path_homo}
Two paths $\bff_1$ and $\bff_2$ are path homotopic with respect to $\calP$ if one can be continuously deformed into the other without intersecting any points in $\calP$.
\end{defi}

\begin{lemma}\label{lemma:partition_to_curve}
Given a partition of a finite $\calP = (\calP^+, \calP^-)$, there exists a curve parameterized by $f(x, y) = 0$ such that $f(x_p,y_p) > 0$ for all $(x_p, y_p) \in \calP^+$ and $f(x_q, y_q) < 0$ for all $(x_q, y_q) \in \calP^-$.
\end{lemma}

\begin{proof}
Let $||\calP^+|| = M$ and $||\calP^-|| = N$.  Then the conditions $f(x_p,y_p)>0$  for all $(x_p, y_p)\in \calP^+$ and $f(x_q,y_q)<0$ for all $(x_q, y_q)\in \calP^-$ give $M + N$ constraints in total. A polynomial with degree larger than $M + N$ exists that satisfies these conditions since this forms an undetermined linear system. 
\end{proof}

\begin{figure}[ht!]
\centering
\caption{Illustration of a partition of $\calP$ by $f(x,y) = 0$.}
\label{fig:1}
\end{figure}

\begin{proposition}\label{prop:partition_to_homotopy}
The homotopy class of a path $\bff$ with respect to $\calP$ uniquely determines a partition of $\calP$ and vice versa. Namely, (a) if two paths are homotopic, then they give the same partition of $\calP$, and (b) if two paths give the same partition of $\calP$, then they are homotopic.
\end{proposition}

\begin{proof}
Proof of (a): Prove by contradiction.
Recall that two paths $\bff_1$ and $\bff_2$ with the same starting and ending points are homotopic if and only if one can be continuously deformed into the other without intersecting any points from $\calP$. Let homotopic paths $\bff_1$ and $\bff_2$ be parameterized by curves $f_1(x, y) = 0$ and $f_2(x, y) = 0$, which partition $\calP$ into $(B^+, B^-)$ and $(C^+, C^-)$ respectively. Assume that $B^+ \neq C^+$ and $B^- \neq C^-$, there exists a point $p = (x_p, y_p)$ such that $p \in B^+$ and $p \notin C^+$, which implies $p \in B^+$ and $p \in C^-$. Therefore, $f_1(x_p, y_p) > 0$ and $f_2(x_p, y_p) < 0$. Since $f_1$ and $f_2$ are homotopic, let $H$ be a continuous map from $f_1$ to $f_2$ $H: \mathbb{R}^2 \times [0, 1] \rightarrow \mathbb{R}$, such that
\begin{equation}
  \begin{split}
  H(x, y, 0) = f_1(x, y)\\
  H(x, y, 1) = f_2(x, y)
  \end{split}
\end{equation}
for all $x$, $y$ in the domain of $f_1$ and $f_2$. Then at $x = x_p$ and $y = y_p$,
\begin{equation}
  \begin{split}
  H(x_p, y_p, 0) = f_1(x_p, y_p) > 0\\
  H(x_p, y_p, 1) = f_2(x_p, y_p) < 0
  \end{split}
\end{equation}
Because $H$ is continuous, there exists some $T \in [0, 1]$, such that $H(x_p, y_p, T) = 0$. Thus contradiction occurs.

Proof of (b): We will prove the contrapositive. Consider two homotopically distinct paths $\bff_1$ and $\bff_2$, one cannot continuously deform into the other without intersecting $\calP$. Then for any continuous map $H$ from $f_1$ to $f_2$, there exists a point $p=(x_p,y_p) \in \calP$ and $T \in [0, 1]$ such that
\begin{equation}
  \begin{split}
  H(x_p, y_p, T) = 0
  \end{split}
\end{equation}
and
\begin{equation}
  \begin{split}
  H(x_p, y_p, 0)H(x_p, y_p, 1) < 0
  \end{split}
\end{equation}
Otherwise, $f_1$ can continuously deform to $f_2$ without intersecting $p$. Therefore, $f_1(x_p, y_p)f_2(x_p, y_p) < 0$. Let $f_1(x, y) = 0$ and $f_2(x, y) = 0$ partition $\calP$ into $(B^+, B^-)$ and $(C^+, C^-)$ respectively, then $p$ cannot be both in $B^+$ and $C^+$ (or $B^-$ and $C^-$). Thus the partition of $\calP$ by $\bff_1$ and $\bff_2$ is different.
\end{proof}

\begin{defi}[Triangulation $\triangulation$ of  $\calP$]\label{def:triangulation}
A triangulated graph $\triangulation = (V, E)$, where $V = \calP$ and $E$ gives a maximal set of non-crossing edges between points of $\calP$
\end{defi}
Fig. \ref{fig:8} shows an example of the triangulation $\triangulation = (V, E)$ of $\calP$, where $V$ and $E$ are shown as grey filled circles and line segments, respectively.

\begin{defi}[A cut/cut-set of $\triangulation$]\label{def:cut}
A cut of $\triangulation = (V, E)$ is a partition of $V$ into two disjoint subsets. A cut-set is the set of edges that have one endpoint in each subset.
\end{defi}

Fig. \ref{fig:10} gives an example of cut of $\triangulation$. Vertices of $\triangulation$ are partitioned into two disjoint subsets distinguished by red and blue circles. Edges in the cut-set are shown by the dashed line segments.

\begin{figure}[!ht]
	\centering
	\subfloat[]{\includegraphics[width=0.3\linewidth]{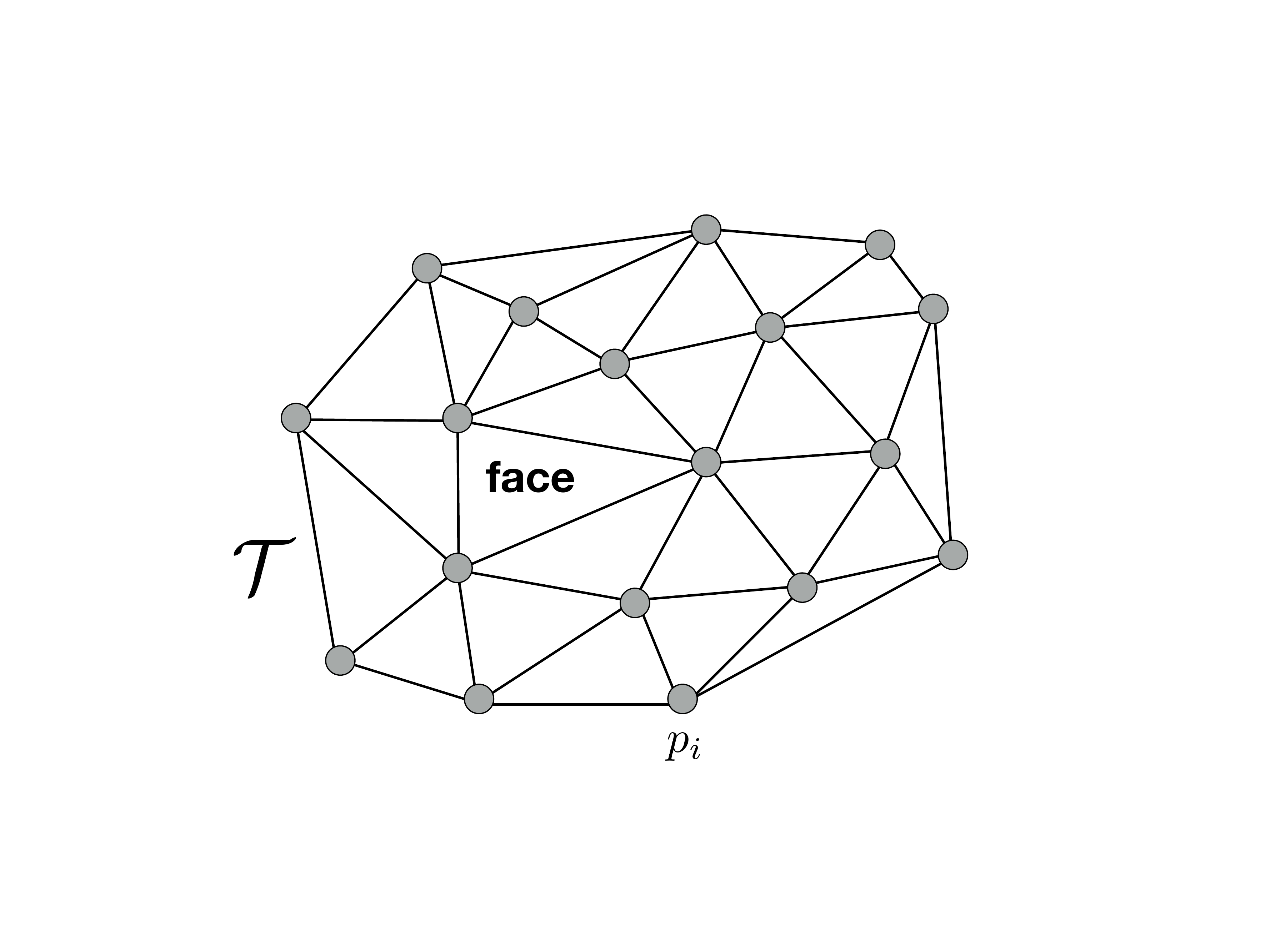}\label{fig:8}}
  ~ ~
  \subfloat[]{\includegraphics[width=0.3\linewidth]{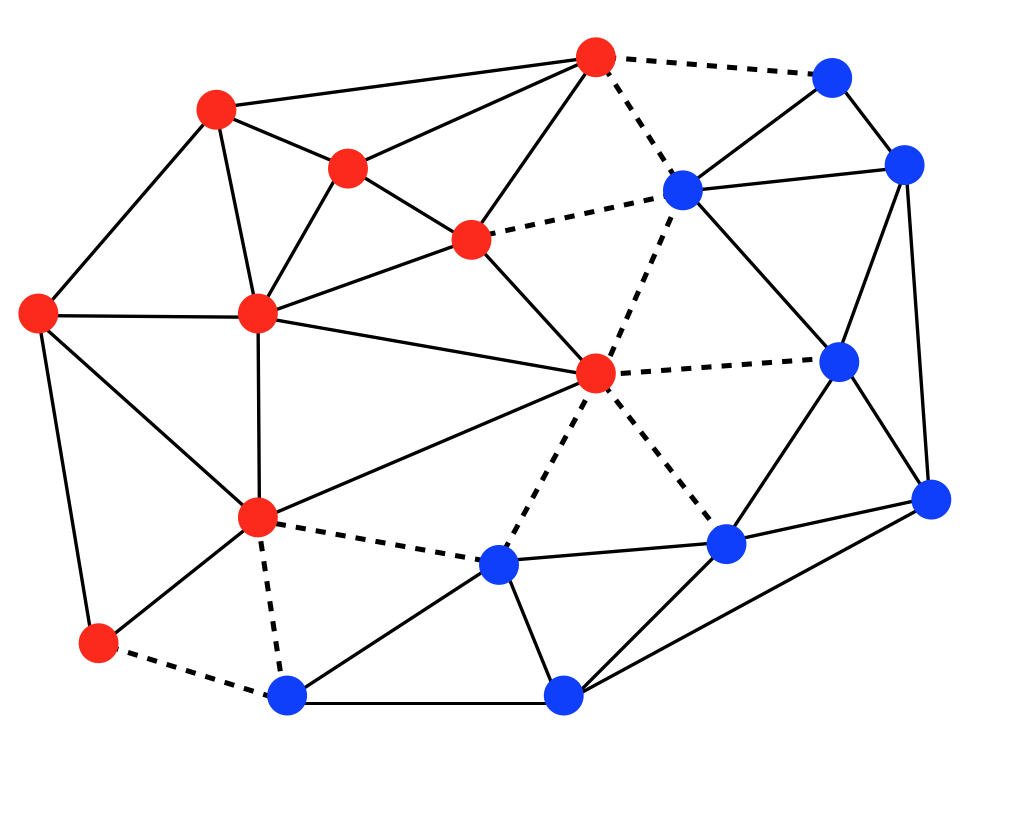}\label{fig:10}}
	\label{fig:81012}
	\vspace{-15pt}
\end{figure}

\begin{proposition}\label{prop:cut_to_homotopy}
A cut of $\triangulation$ uniquely determines a homotopy class of a path through $\calP$.
\end{proposition}

\begin{proof}
Let a cut of $\triangulation$ partition $\calP$ into two disjoint subsets $A$ and $B$. Let $\calP^+ = A$ and $\calP^- = B$, according to Lemma \ref{lemma:partition_to_curve}, we can find a curve parameterized by $f(x, y) = 0$ that gives the partition $\calP = (\calP^+, \calP^-)$. Because of proposition \ref{prop:partition_to_homotopy}, the partition is equivalent to a homotopy class of paths. Therefore, a cut of $\triangulation$ uniquely determines a homotopy class.
\end{proof}

\begin{defi}[Face of $\triangulation$]
A \emph{face} of $\triangulation$ is any region bounded by edges, as illustrated in Fig.~\ref{fig:8}.
\end{defi}

\begin{defi}[Dual graph $\dualgraph$ on $\triangulation$]
The dual graph $\dualgraph = (\dualv,\duale)$ has a vertex for each face of $\triangulation$ and an edge connecting two vertices if the faces in $\triangulation$ are adjacent. Note that there is a vertex $v_{out}$ representing the outer face of the convex hull of $\calP$.
\end{defi}
Fig.~\ref{fig:11} gives an example of a dual graph $\dualgraph $ on $\triangulation$, where $\dualv$ and $\duale$ are shown in red circles and line segments. The blue circle shows the $v_{out}$. Note $s$ and $t$ are in the outer face represented by $v_{out}$ since they are outside of the convex hull of $\calP$.

\begin{defi}[A walk $w$ on $\dualgraph$]
Given $\dualgraph = (\dualv,\duale)$, a walk $w$ on an $\dualgraph$ is a finite alternating sequence of vertices and edges.  
\end{defi}

\begin{defi}[$s \mhyphen t$ cycle on $\dualgraph$]
An $s \mhyphen t$ cycle on $\dualgraph$ is a walk $w$ on $\dualgraph$ with the starting and ending vertex $v_{out}$ representing the outer face. Note that there is no repeated vertices on the walk except for the starting and ending vertex.
\end{defi}

Fig. \ref{fig:9} gives an example of $s \mhyphen t$ on $\dualgraph$ starting and ending at $v_{out}$.

\begin{defi}[Channel]\label{def:channel}
A channel is a series of triangles uniquely determined by a loopless walk on $\dualgraph$.
\end{defi}

Fig. \ref{fig:12} gives an example of a channel determined by a walk on $\dualgraph$. Note that the starting and ending triangles are both represented by $v_{out}$.

\begin{figure}[t]
	\centering
	\subfloat[]{\includegraphics[width=0.35\linewidth]{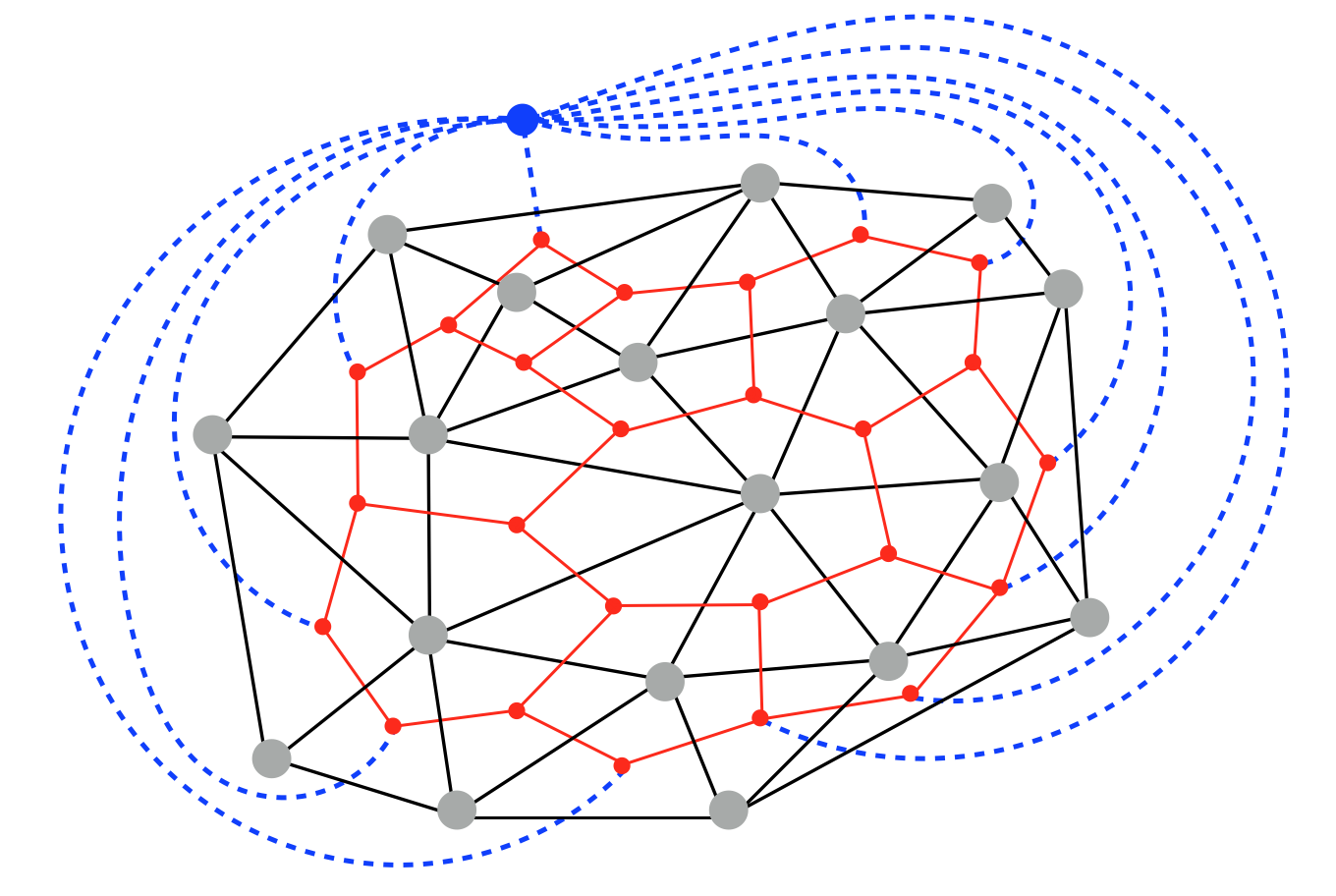}\label{fig:11}}
  \subfloat[]{\includegraphics[width=0.35\linewidth]{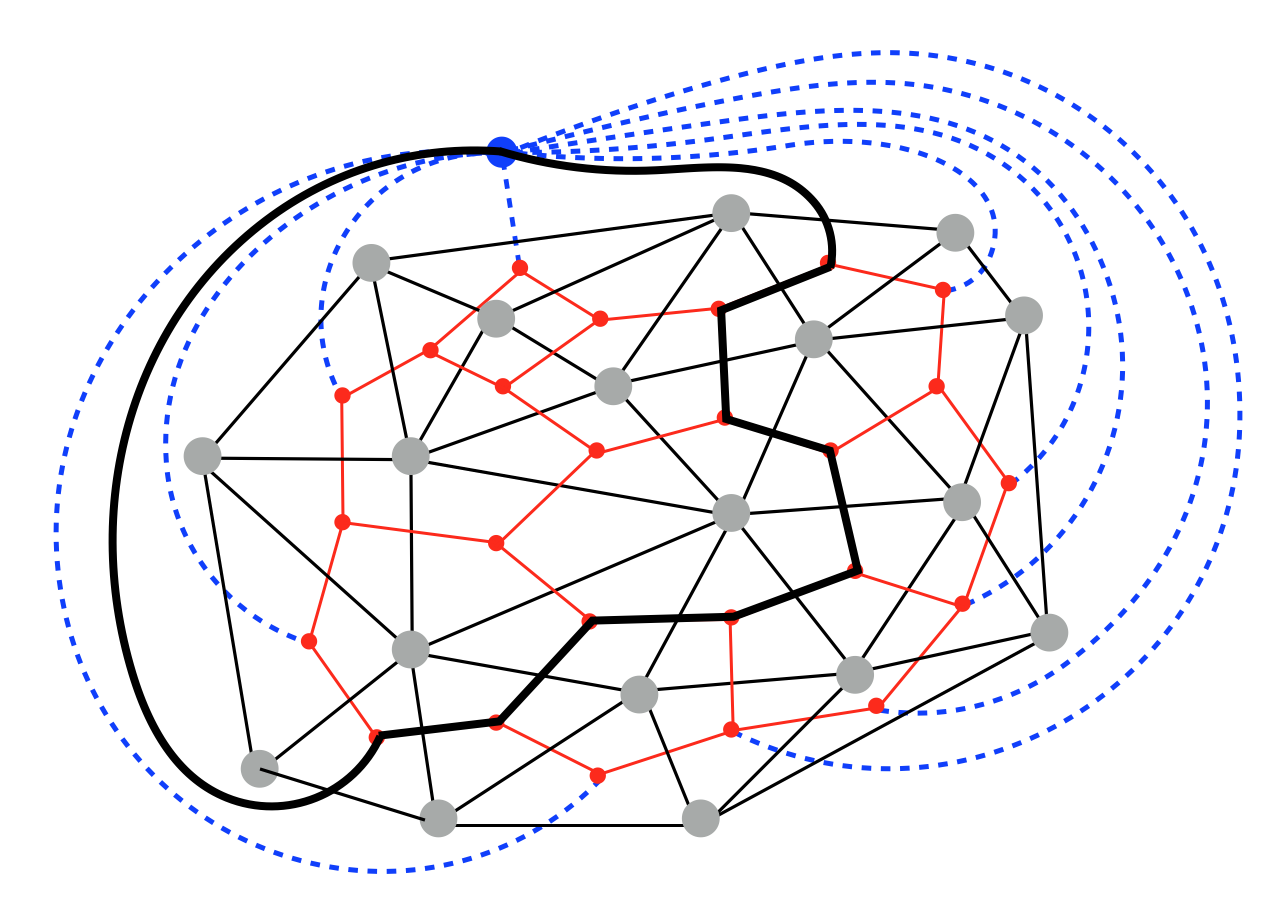}\label{fig:9}}
  \subfloat[]{\includegraphics[width=0.3\linewidth]{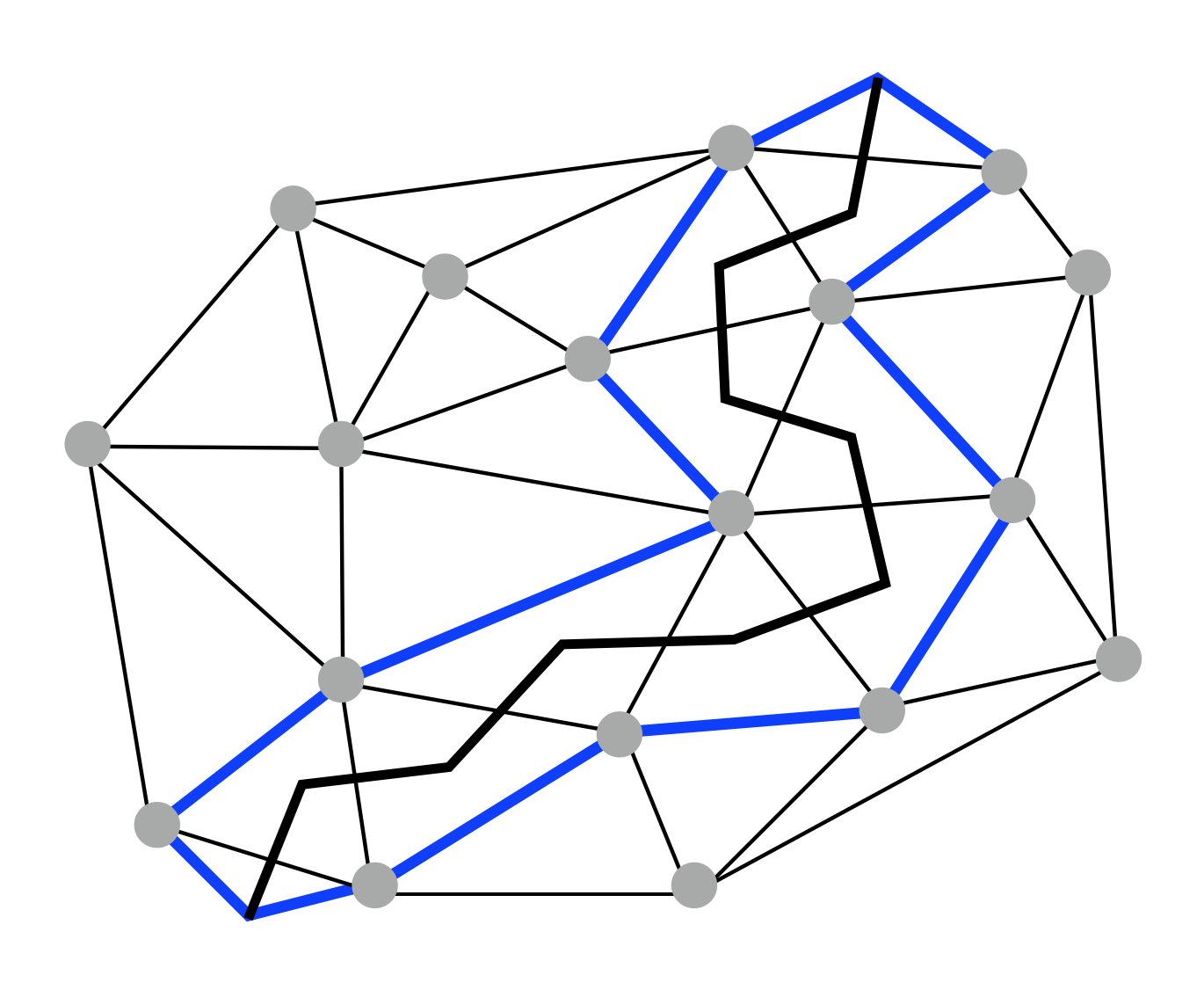}\label{fig:12}}
	\caption{(a) Dual graph in red.  (b) A loopy walk through $\dualgraph$.  (c)  A channel (interior of blue lines).}
	\label{fig:11912}
	\vspace{-15pt}
\end{figure}

\begin{proposition}\label{prop:cycle_to_homotopy}
An $s \mhyphen t$ cycle $\dualgraph$ uniquely determines a homotopy class of a path through $\calP$.
\end{proposition}

\begin{proof}
Because of the duality, a cycle in $\dualgraph$ corresponds to a cut in $\triangulation$ and vice versa. According to Proposition \ref{prop:cut_to_homotopy}, an $s \mhyphen t$ cycle on $\dualgraph$ uniquely determines a homotopy class of paths through $\calP$.
\end{proof}

\begin{theorem}\label{theorem:channel_to_homotopy}
A channel on $\triangulation$ uniquely determines a homotopy class of a path through $\calP$.
\end{theorem}

\begin{proof}
By Definition \ref{def:channel} and Proposition \ref{prop:cycle_to_homotopy}.
\end{proof}

\begin{theorem}
\label{thrm:paths-and-walks}
Every path through $\calP$ uniquely determines a walk on $\dualgraph$
\end{theorem}

\begin{proof}
Let $\bff : [0, 1] \rightarrow \mathbb{R}^2$ be a path intersecting the edges of $\triangulation$ at intersection points $\{i_0, i_1, \hdots, i_n\}$. Let $I_{e_i} = \{i_i, i_{i+1}, \hdots, i_{i+m-1}\}$ be $m$ consecutive intersecting points with the same edge $e_i$ (see Fig.~\ref{fig:13}). Given $||I_{e_i}|| = m$, $I_{e_i}$ can be reduced to $I'_{e_i}$, where $||I'_{e_i}|| = m\mod 2$ without changing the homotopy class of $\bff$. Namely, if $m$ is odd, then $I_{e_i}$ can be reduced to $I'_{e_i} = \{i_{i + j}\}, j\in[0, m-1]$ and if $m$ is even, $I_{e_i}$ can be reduced to $I'_{e_i} = \emptyset$. This is because that for any two consecutive intersection points $i_j$ and $i_{i+1}$ with the same edge, the line segment $e_{i_{i_j \rightarrow i_{j+1}}}$ and path segment $\bff_{i_{i_{j+1} \rightarrow i_j}}$ forms a loop that contains no vertices from $\calP$ inside, otherwise intersections with other edges will occur between $i_j$ and $i_{j+1}$ because vertices are all connected. Therefore, a path $\bff$ with the set of intersection points given by $I_{\bff}$ is homotopically equivalent to $\bff'$ with $I_{\bff'} = I_{\bff} - \{i_j, i_{j+1}\}$. Thus the reduction of every $I_{e_i}$ described above is legitimate. An example is illustrated by Fig. \ref{fig:1314}. In Fig. \ref{fig:13} when $m$ is even, the path is equivalent to never intersecting with edge $e_i$. In Fig. \ref{fig:14}, when $m$ is odd, the path is equivalent to only intersecting with edge $e_i$ once.

As shown above, a path $\bff$ can be reduced to $\bff'$ such that every two consecutive intersections of $\bff'$ are with distinct edges of $\triangulation$. For each intersection points $i_j \in I_{\bff}$, an edge in $\dualgraph$ can be uniquely determined. For every two consecutive intersection points $i_j, i_{j+1} \in I_{\bff}$, the triangle that both edges belong to (a vertex in $\dualgraph$) can be uniquely determined. Therefore, consecutive edges of $\dualgraph$ given by $I_{\bff}$ are connected with each other, which forms unique a walk on $\dualgraph$.
\end{proof}

\begin{figure}[!ht]
	\centering
	\subfloat[$m$ is even.]{\includegraphics[width=0.5\linewidth]{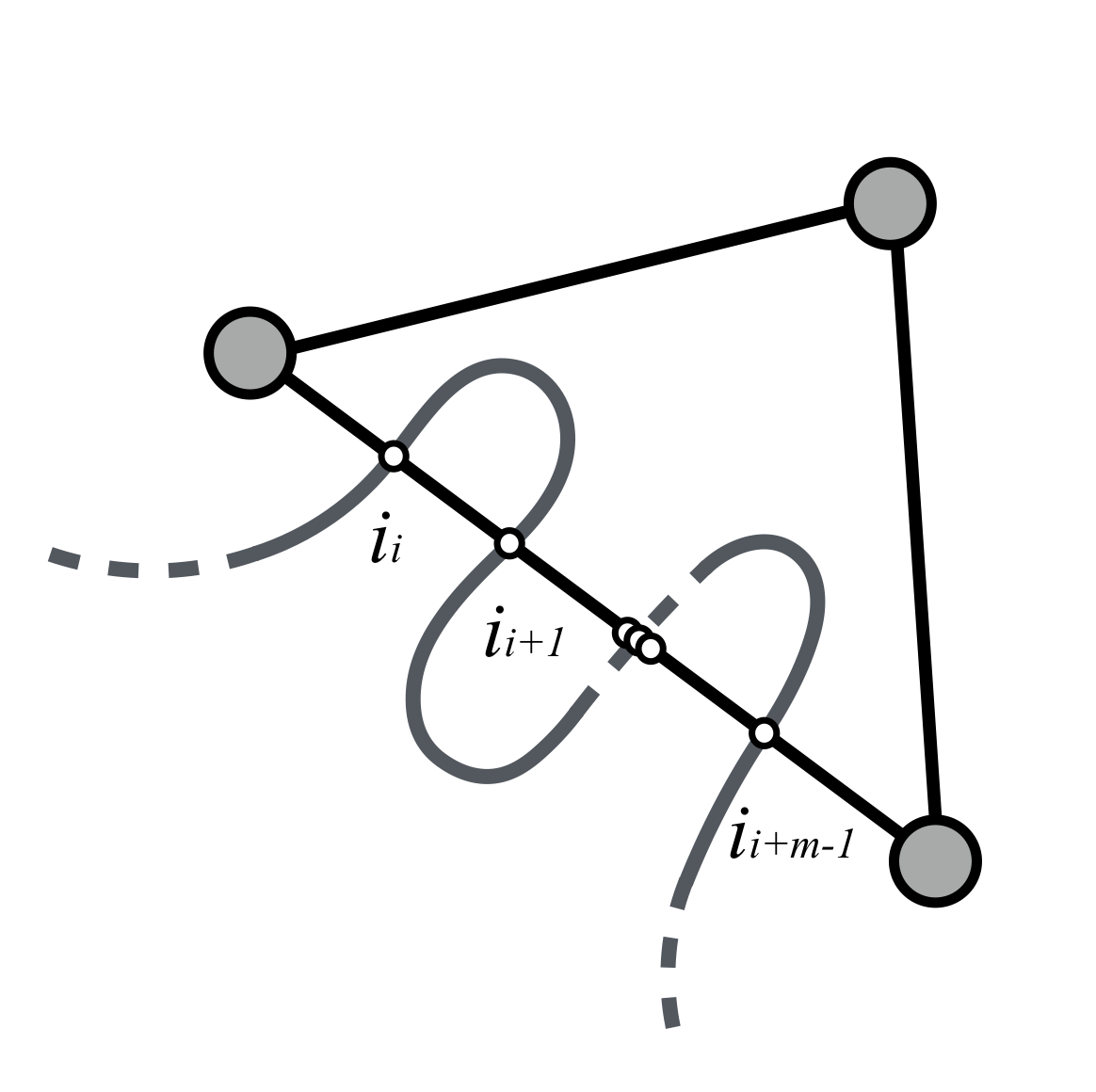}\label{fig:13}}
  \subfloat[$m$ is odd.]{\includegraphics[width=0.5\linewidth]{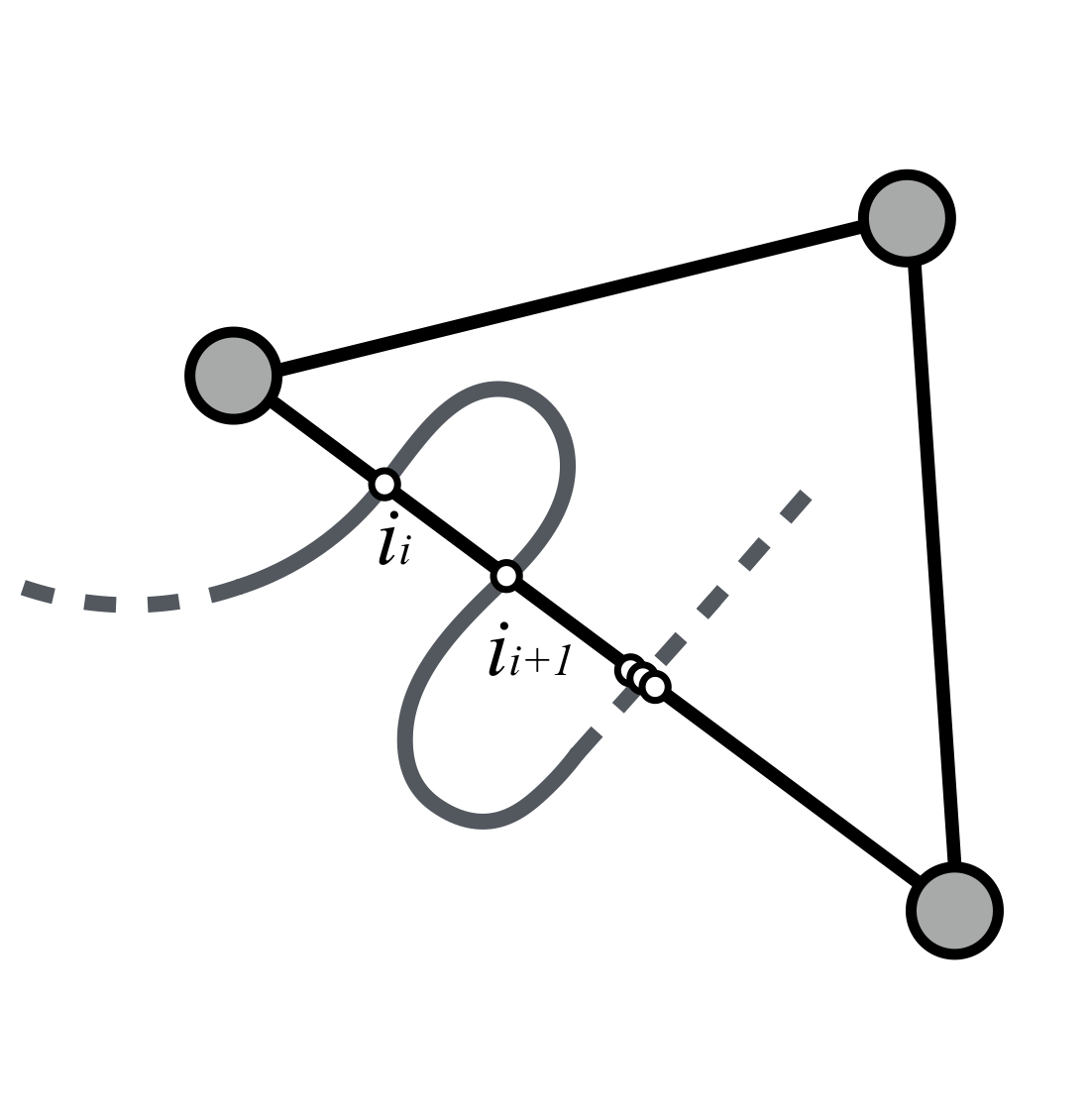}\label{fig:14}}
  	\caption{Illustration for Theorem~\ref{thrm:paths-and-walks}}
	\label{fig:1314}
	\vspace{-15pt}
\end{figure}

\bibliographystyle{IEEEtran}
\bibliography{icra2019}

\end{document}